\pdfoutput=1
\documentclass[11pt]{article}

\usepackage[letterpaper,margin=1in]{geometry}
\usepackage[T1]{fontenc}
\usepackage[utf8]{inputenc}
\usepackage{lmodern}
\usepackage{times}
\usepackage[round,authoryear]{natbib}

\usepackage{xcolor}
\usepackage{bm}
\usepackage{mathtools}
\usepackage{enumitem}
\setlist{topsep=4pt plus 1pt minus 2pt, partopsep=1pt plus 0.5pt minus 0.5pt,
  itemsep=2pt plus 1pt minus 0.5pt, parsep=2pt plus 1pt minus 0.5pt}
\usepackage{booktabs}
\usepackage{microtype}
\usepackage{graphicx}
\usepackage{url}

\usepackage{amsmath,amssymb,amsthm,amsfonts}

\usepackage{amsmath,amsfonts,bm}

\def\1{\bm{1}}

\DeclareMathAlphabet{\mathsfit}{\encodingdefault}{\sfdefault}{m}{sl}
\SetMathAlphabet{\mathsfit}{bold}{\encodingdefault}{\sfdefault}{bx}{n}

\newcommand{\E}{\mathbb{E}}

\newcommand{\R}{\mathbb{R}}

\newtheorem{theorem}{Theorem}[section]
\newtheorem{proposition}[theorem]{Proposition}
\newtheorem{lemma}[theorem]{Lemma}
\newtheorem{corollary}[theorem]{Corollary}
\newtheorem{definition}[theorem]{Definition}
\newtheorem{assumption}{Assumption}
\newtheorem{example}[theorem]{Example}
\newtheorem{remark}[theorem]{Remark}

\newcommand{\N}{\mathbb{N}}

\newcommand{\Om}{\Omega}
\newcommand{\PO}{\mathcal{P}(\Om)}                     %
\newcommand{\PPO}{\mathcal{P}(\mathcal{P}(\Om))}       %
\newcommand{\bbP}{\mathbb{P}}
\newcommand{\bbQ}{\mathbb{Q}}

\newcommand{\dd}{\,\mathrm{d}}
\newcommand{\MMD}{\mathrm{MMD}}
\newcommand{\SW}{\mathrm{SW}}
\newcommand{\VSW}{\mathrm{VSW}}
\newcommand{\W}{\mathrm{W}}
\newcommand{\WoW}{\mathbb{W}}
\newcommand{\Hk}{\mathcal{H}_{\kappa}}
\newcommand{\ip}[2]{\langle #1,#2\rangle}
\newcommand{\norm}[1]{\left\lVert #1 \right\rVert}
\newcommand{\pf}{\#}
\newcommand{\Lip}{\mathrm{Lip}}
\newcommand{\supp}{\mathrm{supp}}
\newcommand{\law}{\mathrm{law}}
\newcommand{\Cont}{C(\Om)}

\usepackage{algorithm}
\usepackage{algpseudocode}

\usepackage{multirow}
\usepackage{wrapfig}

\definecolor{linkblue}{RGB}{0,64,160}
\usepackage[colorlinks=true,linkcolor=linkblue,citecolor=linkblue,urlcolor=linkblue,
  pdftitle={Vertical Sliced Wasserstein Distances for Meta-Measures}]{hyperref}

\title{\textbf{Vertical Sliced Wasserstein Distances\\ for Meta-Measures}}

\usepackage{authblk}

\author[1]{Elaheh Akbari}
\author[2]{Ahmadreza Moradipari}
\author[2]{Kyungtae Han}
\author[3]{Hamed Pirsiavash}
\author[4]{\protect\\{Yikun Bai}}
\author[5]{Rocio Diaz Martin}
\author[1]{Soheil Kolouri}

\affil[1]{College of Connected Computing, Vanderbilt University}
\affil[2]{Toyota Motor North America, InfoTech Labs}
\affil[3]{Department of Computer Science, University of California, Davis}
\affil[4]{Department of Mathematics, Purdue University}
\affil[5]{Department of Mathematics, Florida State University}
\affil[ ]{\vspace{0.4em} \texttt{elaheh.akbari@vanderbilt.edu}}

\date{}

\begin{document}

\maketitle

\begin{abstract}
\noindent Many modern learning problems require comparing probability distributions whose individual observations are themselves probability measures, naturally giving rise to \emph{meta-measures} in $\PPO$. However, optimal transport distances on this space can be prohibitively expensive to evaluate. We introduce the \emph{vertical sliced Wasserstein} ($\VSW$) distance, a scalable metric on $\PPO$. Rather than slicing along the transport geometry of $\mathcal{P}(\Om)$, $\VSW$ projects each inner measure using continuous test functions and compares the resulting one-dimensional distributions with one-dimensional Wasserstein distance. We prove that $\VSW$ is a valid metric on $\PPO$, metrizes weak convergence, and admits dimension-independent statistical convergence guarantees. Empirically, $\VSW$ provides high accuracy and efficiency compared with existing distances for meta-measures. We further derive a practical particle gradient for the empirical $\VSW_2^2$ objective and demonstrate its use for gradient-based transport of meta-measures. $\VSW$ provides a simple and scalable alternative for comparing and optimizing distributions of distributions.

\end{abstract}

\section{Introduction and Related Work}
\label{sec:intro}

Many learning problems naturally involve \emph{distributions of distributions} (DoDs): a dataset can be viewed as a distribution over class-conditional distributions \citep{alvarez2020otdd}, a collection of shapes as a distribution over point clouds where each point cloud is itself an empirical distribution \citep{haviv2024wasserstein}, and documents, cell populations, or mixture models can similarly be represented as collections of probability measures \citep{yurochkin2019hierarchical,nguyen2024summarizing}. Formally, these objects are \emph{meta-measures} $\bbP\in\PPO$, i.e., probability measures whose observations are themselves probability measures on $\Om$. Comparing such objects requires accounting for variation both \emph{within} each inner distribution and \emph{across} the population of distributions.

\noindent\textbf{Wasserstein distances on meta-measures.} A natural metric for this setting is Wasserstein-over-Wasserstein (WoW), which solves an outer optimal transport (OT) problem using Wasserstein distance between inner measures as its ground cost.  Related constructions appear in hierarchical OT \citep{yurochkin2019hierarchical}, generative modeling \citep{dukler2019wasserstein}, and the optimal transport dataset distance (OTDD) \citep{alvarez2020otdd}.  While geometrically appealing, WoW is expensive: comparing empirical meta-measures requires many inner OT problems and inherits the statistical and computational difficulties of Wasserstein distance at both levels \citep{dukler2019wasserstein,bonet2025flowing}. Additionally, \citet{bonet2025flowing} develop WoW gradient flows and introduce tractable MMD objectives with Sliced-Wasserstein kernels.

\noindent\textbf{Sliced and scalable meta-measure distances.}
Recent methods therefore seek scalable alternatives through slicing. Double-Sliced Wasserstein (DSW) \citep{dsw} uses one-dimensional Wasserstein projections and quantile representations, Busemann-based approaches slice along geodesic directions of Wasserstein space \citep{bonet2025busemann}, and s-OTDD uses moment-based projections for dataset comparison \citep{sotdd}. Although substantially cheaper than exact WoW, these approaches still construct their projections from the transport geometry of the inner measures or surrogates thereof.

\textbf{We introduce \emph{Vertical Sliced Wasserstein} (VSW)}, a complementary approach that slices a measure through its action on continuous test functions. For a test function $f\in\Cont$, an inner measure $\mu$ is represented by the scalar
$P_f(\mu)=\int_\Om f\,\dd\mu$.
Applying this projection to every inner measure converts a meta-measure into a one-dimensional distribution. Then, one can compare efficiently $\bbP$ and $\bbQ$ in $\mathcal{P}(\mathcal{P}(\Om))$ using the one-dimensional $p$-Wasserstein distance $\W_p$ by \(
\VSW_p^p(\bbP,\bbQ)
:=
\E_{f\sim\sigma}
[
\W_p^p\big((P_f)_\pf\bbP,(P_f)_\pf\bbQ\big)
],
\) where $\sigma$ is the law of a Gaussian process with universal covariance kernel $\kappa$. We call this construction \emph{vertical} because it acts through the dual representation of each measure, in contrast to methods that slice along its transport geometry.

This simple change has important practical consequences. Each slice compresses an entire inner distribution to one scalar, requiring only test-function evaluation and averaging before a single one-dimensional OT computation across the outer distributions. VSW therefore avoids inner OT, per-distribution quantile representations, and Gaussian surrogates. At the same time, the family of test-function projections is rich enough to distinguish meta-measures: we prove that $\VSW_p$ is a metric and metrizes weak convergence. For $p=2$, VSW is additionally Hilbertian and closely connected to kernel mean embeddings, while retaining the outer Wasserstein comparison that distinguishes distributions beyond their mean measures.

Empirically, this combination yields an accurate and efficient meta-measure distance. Across dataset comparison, point-cloud and image-patch distributions, and shape classification, VSW closely tracks transport-based reference distances while consistently reducing runtime. %
We further derive an explicit gradient formulation for the empirical $\VSW_2^2$ objective, supporting parametric
generative modeling. %
Our main contributions are:
\begin{enumerate}[leftmargin=15pt]
    \item We introduce \textbf{Vertical Sliced Wasserstein (VSW)}, a new sliced distance for meta-measures based on test-function projections, and prove that it defines a metric that metrizes weak convergence.

    \item We establish statistical and computational scalability, including dimension-independent two-stage convergence rates and efficient Monte Carlo estimation without inner OT.

    \item We show that $\VSW_2$ is Hilbertian, induces universal characteristic kernels on meta-measures, and admits, under discretization, an explicit gradient for optimization and generative modeling.

    \item We provide broad empirical validation across dataset comparison, point clouds, image patches, shape classification, discrete particle gradient flows, and hierarchical generation, where VSW is fast while maintaining strong agreement with transport-based baselines.
\end{enumerate}

\enlargethispage{\baselineskip}
The closest computational baseline is DSW  \citep{dsw}, which first projects multivariate inner measures to one dimension, represents the resulting measures through quantile functions in $L^2([0,1])$, and then applies functional slicing with Gaussian-process directions. VSW instead applies test functions directly to the original inner measures, reducing each to a single scalar per slice and avoiding both inner quantile construction and an additional Euclidean slicing stage.
A related test-functional approach is the hierarchical IPM of \citet{catalano2024hierarchical}, which, unlike VSW, requires optimization over a class of test functions.%

\section{Background and standing assumptions}
\label{sec:background}

\paragraph{Spaces.} Throughout, $(\Om,\rho)$ is a compact metric space. When gradients on $\Om$ are needed (Sections~\ref{sec:computation}--\ref{sec:flows}) we take $\Om\subset\R^d$. $\Cont$ denotes the continuous real functions with the sup norm $\norm{\cdot}_\infty$. Since $\Om$ is compact, $(\Cont,\|\cdot\|_\infty)$ is a separable Banach space. By the Riesz--Markov representation theorem its dual is the space of finite signed Radon measures $\mathcal{M}(\Om)$ with the total variation norm, and $\PO$ is a convex subset of the unit ball of $\mathcal{M}(\Om)$. The weak (narrow) topology on $\PO$ is the restriction of the weak$^*$ topology of $\mathcal{M}(\Om)=\Cont^*$.

By Prokhorov's theorem, $\PO$ is weakly compact 
and it is metrizable because $\Cont$ is separable. Its Borel $\sigma$-algebra, $\mathcal{B}(\PO)$, is generated by the maps $P_f:\PO\to \R$ 
\begin{equation}\label{eq:vert-proj}
    P_f:\mu\mapsto\ip{f}{\mu}:=\int_\Om f\dd\mu,\qquad f\in\Cont,
\end{equation}
\citep[Lem.~4.5 and Thm.~4.2]{kallenberg2017random}.
Since $\PO$ is itself a compact metric space, the same reasoning applies one level up: $\PPO$, the space of Borel probability measures on $\PO$, is compact and metrizable in its weak topology. For $\bbP\in\PPO$ we write $\mu\sim\bbP$ for a random inner measure. %

Note the two elementary but crucial algebraic properties of \emph{vertical projections} defined by \eqref{eq:vert-proj}:
\begin{lemma}%
\label{lem:algebraic}
For all $f,g\in\Cont$, $a,b\in\R$, $\mu,\mu_0,\mu_1\in\PO$, and $t\in[0,1]$,
\[
P_{af+bg}(\mu)=aP_f(\mu)+bP_g(\mu),
\qquad
P_f((1-t)\mu_0+t\mu_1)
=(1-t)P_f(\mu_0)+tP_f(\mu_1).
\]
Thus, $P_f(\mu)$ is linear in $f$ and affine in $\mu$ along mixtures.
\end{lemma}

\noindent\textbf{Optimal transport between one-dimensional measures.} For $p\ge1$ and $\alpha,\beta\in\mathcal{P}(\R)$, with finite $p$-moment, $1\leq p<\infty$, let  $F_\alpha,F_\beta$ denote their corresponding Cumulative Distribution Functions (CDFs), and $q_\alpha=F_\alpha^{-1},q_\beta=F_\beta^{-1}$ be their quantile functions. The optimal transport cost is:
\begin{equation}
\label{eq:1dot}
\W_p^p(\alpha,\beta)=\int_0^1\big|q_\alpha(u)-q_\beta(u)\big|^p\dd u.
\end{equation}
We use $\W_p^p(\alpha,\beta)\le(\operatorname{diam}\supp)^{p-1}\W_1(\alpha,\beta)$ for compactly supported $\alpha,\beta$, and the empirical rate: if $\alpha$ is supported in an interval of length $|I|$ and $\widehat\alpha_N$ is the empirical measure of $N$ i.i.d.\ samples, then
\begin{equation}
\label{eq:1drate}
\E\,\W_1(\widehat\alpha_N,\alpha)\;%
\le\;\tfrac{1}{\sqrt N}\int_\R\sqrt{F_\alpha(1-F_\alpha)}\;\le\;\tfrac{|I|}{2\sqrt N},
\end{equation}
where the first inequality follows from \citet[Thm 3.2]{bobkov2019one}.%

\noindent\textbf{Gaussian slicing measure.} 
We write $\sigma$ for a Borel probability measure on $\Cont$. A slicing direction is a realization $f\sim\sigma$, and $L$ {\text{iid}} slices are $f_1,\dots,f_L{\sim}\sigma$. This is akin to the sliced-Wasserstein distance with Gaussian projections introduced by \cite{nadjahi2021fast}, but for meta-measures.

\begin{assumption}%
\label{ass:kernel}
We consider a continuous, symmetric, positive definite kernel and universal ``slicing kernel'' $\kappa:\Om\times \Om\to \R$, where universal means that  the reproducing kernel Hilbert space (RKHS) $\Hk$ associated to $\kappa$ is dense in $(\Cont,\|\cdot\|_\infty)$, and such that the centered Gaussian process with covariance kernel $\kappa$, $\mathrm{GP}(0,\kappa)$, admits
a version with almost surely continuous sample paths. 
\end{assumption}

Under the above assumption, we consider the model where $f\sim\sigma=\mathrm{GP}(0,\kappa)$. For every finite collection $x_1,\dots, x_n\in \Om$, the random vector $(f(x_1),\dots,f(x_n))$ has law $\mathcal{N}(0, K)$, where 
$K_{ij}=\E_{f\sim \mathrm{GP}(0,\kappa)}[f(x_i)f(x_j)]=\kappa(x_i,x_j)$. 
 See extra
details  in Appendix~\ref{app:background_app}.

\noindent\textbf{Wasserstein-over-Wasserstein.}
For a Borel metric $D$ on $\PO$,
\[
\WoW_{p,D}^p(\bbP,\bbQ)
:=
\inf_{\Gamma\in\Pi(\bbP,\bbQ)}
\int D(\mu,\nu)^p\,\dd\Gamma(\mu,\nu),
\]
with $\Pi(\bbP,\bbQ)$ the set of couplings between $\bbP,\bbQ\in\PPO$. When $D=\W_p$, it recovers the standard
WoW distance, which we simply denote by $\WoW_p$.
Additional details of this section are available in Appendix~\ref{app:background_app}.

\section{Our method: The Vertical Sliced Wasserstein distance}
\label{sec:vsw}

\begin{definition}%
\label{def:vsw}
For $\bbP,\bbQ\in\PPO$, $p\in[1,\infty)$, and $\sigma=\mathrm{GP}(0,\kappa)$ the  Gaussian slicing measure under Assumption~\ref{ass:kernel}, we define the Vertical Sliced Wasserstein discrepancy by
\begin{equation}
\label{eq:vsw}
\VSW_p(\bbP,\bbQ)\;:=\;\Big(\int_{\Cont}\W_p^p\big((P_f)_\pf\bbP,(P_f)_\pf\bbQ\big)\dd\sigma(f)\Big)^{1/p},
\end{equation}
where, for each fixed $f\in\Cont$, the pushforward $(P_f)_\pf\bbP\in\mathcal{P}(\R)$ is the  law of the scalar vertical slice $P_f(\mu)$ when $\mu\sim\bbP$ defined precisely by $(P_f)_\pf\mathbb{P} (B)=\bbP(\{\mu\in\PO:P_f(\mu)\in B\})$ for all $B\in \mathcal{B}(\R)$, which is supported on $[-\norm{f}_\infty,\norm{f}_\infty]$ %
(similarly for $(P_f)_\pf\bbQ$).
\end{definition}

The \emph{``vertical''} viewpoint reflects the affine structure
of the space of measures: the slicers \eqref{eq:vert-proj} are
affine along mixtures $(1-t)\mu_0+t\mu_1$, as stated in
Lemma~\ref{lem:algebraic}. In contrast, \emph{``horizontal''}
or \emph{transport geometry} describes the movement of mass
within $\Om$, e.g., through
displacement (McCann) interpolation 
$
\mu_t=((1-t)\mathrm{id}+tT)_\#\mu_0$, for an optimal Monge map $T$ for the quadratic-cost.
A related distinction between spatial transport and local mass variation
appears in unbalanced OT \citep{chizat2018interpolating,liero2018optimal};
here ``vertical'' specifically refers to affinity along probability mixtures.

\subsection{Well-posedness, metricity, and topology of VSW}
We show regularity of the slicing procedure to ensure that $\VSW_p$ in \eqref{eq:vsw} is well-defined (Lemma \ref{lem:regularity}). Moreover, $\VSW_p$  defines a true metric on $\mathcal{P}(\PO)$ (Theorem \ref{thm:metric})  which metrizes weak convergence (Proposition \ref{thm:topology}), and $(\mathcal{P}(\PO),\VSW_2)$  embeds isometrically into a  Hilbert space (Theorem \ref{thm:hilbertian}).

\begin{lemma}%
\label{lem:regularity}
(i) The map $P:\Cont\times\PO\to\R$, $P(f,\mu):=\ip{f}{\mu}$, is jointly continuous. (ii) For fixed $\bbP,\bbQ\in\PPO$ the map $f\mapsto \W_p\big((P_f)_\pf\bbP,(P_f)_\pf\bbQ\big)$ is $2$-Lipschitz from $(\Cont,\norm{\cdot}_\infty)$ to $\R$; in particular it is continuous,  and the integral in \eqref{eq:vsw} is well defined and finite. %
\end{lemma}

\begin{theorem}%
\label{thm:metric}
$\VSW_p$ is a metric on $\PPO$ for every $p\in[1,\infty)$.
\end{theorem}

\begin{remark}[Comparison with horizontal slicing]
Busemann projections are not yet known to distinguish meta-measures
in general \citep{bonet2025busemann}, while moment projections need
additional assumptions ensuring that moments determine the measures
\citep{sotdd}. In our setting, vertical slicing
distinguishes meta-measures whenever the kernel is universal,
with no extra assumptions on the meta-measures themselves.
\end{remark}

\begin{proposition}%
\label{thm:topology}
For every $p\in[1,\infty)$, the metric $\VSW_p$ induces the weak topology on $\PPO$. In particular $(\PPO,\VSW_p)$ is a compact metric space.
\end{proposition}

\begin{theorem}%
\label{thm:hilbertian}
Let $\lambda$ be Lebesgue measure on
$(0,1)$. The map
\[
T:\PPO\to L^2(\Cont\times(0,1),\sigma\otimes\lambda), \qquad T(\bbP)(f,u):=q_{(P_f)_\pf\bbP}(u),
\]
is well defined and satisfies
$$
\norm{T(\bbP)-T(\bbQ)}_{L^2(\sigma\otimes\lambda)}
=\VSW_2(\bbP,\bbQ).$$ 
Hence $(\PPO,\VSW_2)$ embeds isometrically into a separable Hilbert space.
Moreover, $\VSW_2^\beta$ is conditionally negative definite for
$\beta\in(0,2]$, so $\exp(-\gamma\VSW_2^\beta)$ is positive definite for
all $\gamma>0$.
\end{theorem}

Further details
are provided in Appendices \ref{sec:comparisons} and \ref{sec: stats}, and all proofs  of this section are in Appendix ~\ref{sec:proofs:vsw}.

\subsection{Computation}
\label{sec:computation}

\paragraph{Sampling slices:} For translation-invariant $\kappa$ on $\Om\subset\R^d$ with spectral measure $\tau_\kappa$ (Bochner), an $R$-feature slice is
\begin{equation}
\label{eq:rff}
f_R(x)=\sqrt{\tfrac{2}{R}}\sum_{r=1}^{R}a_r\cos\big(\ip{\omega_r}{x}+b_r\big),
\qquad a_r\sim\mathcal N(0,1),\ \omega_r\sim\tau_\kappa,\ b_r\sim\mathrm{Unif}[0,2\pi],
\end{equation}
where we normalize $\kappa(x,x)=\kappa(0)=1$ %
so that the spectral measure $\tau_\kappa$ is a probability distribution. %
Conditionally on $(\omega,b)$, $f_R$ is a Gaussian process whose covariance converges uniformly to $\kappa$ as $R\to\infty$ \citep{rahimi2007random}. %
Alternatively, a truncated Karhunen--Lo\`eve expansion may be used (see uniform approximation guarantees in Proposition~\ref{prop:stability} and Lemma~\ref{lem:kl}). Gradients have the closed form $\nabla f_R(x)=-\sqrt{2/R}\sum_r a_r\sin(\ip{\omega_r}{x}+b_r)\,\omega_r$.

\noindent\textbf{Cost:} The full VSW algorithm is shown in Algorithm~\ref{alg:vsw}. Let $N$ denote the maximum number of inner measures in either empirical meta-measure, and $n$ the maximum number of points per inner measure. Evaluating $L$ independent $R$ feature slices costs $O\!\left(LNnRd+LN\log N\right)$. In practice, we share the Fourier basis $(\omega_r,b_r)_{r=1}^R$ across slices and resample only the Gaussian coefficients (conditional on the basis, this estimates Gaussian-sliced Wasserstein between RFF mean embeddings in $\R^R$; see Appendix~\ref{app:finite-features} for the scaling identity and pseudometric caveat). Computing and averaging the $R$ basis features once for each inner measure, and reusing them across slices, reduces the cost to $O\!\left(NnRd+LNR+LN\log N\right)$. Unlike prior sliced WoW constructions, VSW requires no per inner-measure sorting or quantile representation. Its runtime advantages are verified empirically in Section~\ref{sec:experiments}.

\subsection{Statistical properties}
\label{sec:statistics}

We consider the two-stage observation model of \citet{szabo2016learning}: inner measures $\mu_1,\dots,\mu_N\stackrel{\text{iid}}{\sim}\bbP$; for each $i$, points $x_{i1},\dots,x_{in}\stackrel{\text{iid}}{\sim}\mu_i$ (conditionally independent given $\mu_i$), empirical inner measures $\widehat\mu_i=\frac1n\sum_j\delta_{x_{ij}}$, and the plug-in meta-measures \( \widehat\bbP_N=\frac1N\sum_{i=1}^N\delta_{\mu_i}, ~ \widetilde\bbP_{N,n}=\frac1N\sum_{i=1}^N\delta_{\widehat\mu_i}. \)

\begin{theorem}
\label{thm:sample}
For every $p\in[1,\infty)$, there exists
$C_p<\infty$, depending only on $p$ (with $C_p=1$ for $p\in[1,2]$), such that
\[
\E\,\left[\VSW_p^p(\bbP,\widetilde\bbP_{N,n})\right]
\le
2^{p-1}\Big(
\tfrac12\,2^p\Lambda_p N^{-1/2}
+
C_p\,2^p\Lambda_p n^{-p/2}
\Big),
\]
where $\Lambda_p=\E_\sigma\norm{f}_\infty^p<\infty$. %
The rate exponents are dimension-free.
Dependence on $(\Om,\kappa)$ enters only through the sup-norm moment, which is finite %
but may grow with the size of $\Om$ relative to the kernel bandwidth.
\end{theorem}

\begin{remark}%
\label{rem:sample-rates}
Unlike $\WoW_{p,\W_p}$, which inherits the dimensional dependence of
Wasserstein estimation at both levels \citep{fournier2015rate,catalano2024hierarchical},
Theorem~\ref{thm:sample}, by Jensen, gives dimension-free rate exponents for
$\E\,\VSW_p(\bbP,\widetilde\bbP_{N,n})$:
$N^{-1/(2p)}$ at the outer level and $n^{-1/2}$ at the inner level,
consistent with sliced-Wasserstein and two-stage sampling rates
\citep{nadjahi2020statistical,szabo2016learning}.
For $p=2$, the worst-case outer distance rate $N^{-1/4}$ is sharp and
slower than HIPM's parametric $N^{-1/2}$ rate; it improves to $N^{-1/2}$
under the projected-density condition in Remark~\ref{rem:outer-rate}.
\end{remark}

Finally, the Monte Carlo slice complexity reads as follows:
\begin{proposition}
\label{prop:mc}
Let $f_1,\dots,f_L\stackrel{\text{iid}}{\sim}\sigma$ and $\widehat{\VSW}^p_L(\bbP,\bbQ):=\frac1L\sum_{l}\W_p^p\big((P_{f_l})_\pf\bbP,(P_{f_l})_\pf\bbQ\big)$. Then
\[
\E\Big|\widehat{\VSW}^p_L-\VSW_p^p\Big|\;\le\;\sqrt{\tfrac{\operatorname{Var}_\sigma\big[\W_p^p\big]}{L}}\;\le\;2^p\sqrt{\Lambda_{2p}/L} .
\]
\end{proposition}

Following \citet{nadjahi2020statistical}, we examine empirical scaling on
Gaussian-mixture meta-measures for $d\in\{2,10,20,50,100\}$ using finite
Fourier features. Figure~\ref{fig:stats} shows an $L^{-1/2}$ trend for
Monte Carlo \emph{distance} error against a finite-slice reference, and
$\VSW_2^2$ trends of $N^{-1/2}$ and $n^{-1}$ at the outer and inner levels,
with nearly unchanged slopes across dimensions. This is an empirical scaling
study: Theorem~\ref{thm:sample} assumes a compact domain, and
Proposition~\ref{prop:mc} bounds the powered-distance estimator.

\begin{figure}[t!]
  \centering
  \includegraphics[width=.9\textwidth]{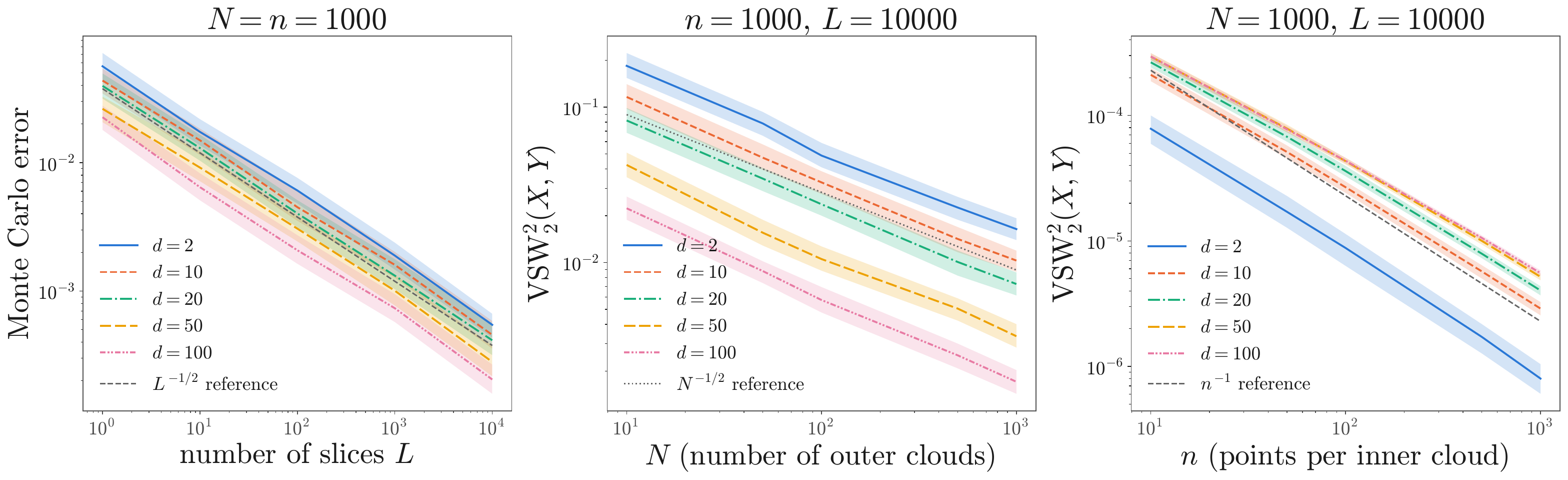}
  \caption{Statistical complexity analysis for empirical meta-measures $X,Y$. Setup in Appendix~\ref{app:impl}. The outer law has three atoms, so the $N^{-1/2}$ reference for $\VSW_2^2$ corresponds to the atomic-outer rate (Remark~\ref{rem:outer-rate}). Left: distance error against a $10^5$-slice reference. Right: $X,Y$ share component labels to isolate inner sampling error.}
  \label{fig:stats}
\end{figure}

\begin{algorithm}[t]
\caption{$\widehat{\VSW}_p^p$ between empirical meta-measures}
\label{alg:vsw}
{\begin{algorithmic}[1]
\Require clouds $\{X_i\}_{i\le N}$, $\{Y_k\}_{k\le M}$; kernel $\kappa$; slices $L$; features $R$; order $p$
\For{$l=1,\dots,L$}
  \State sample $f_l$ via \eqref{eq:rff} (or KL truncation)
  \State $s_i\gets\frac{1}{|X_i|}\sum_{x\in X_i}f_l(x)$ for all $i$;\quad $t_k\gets\frac{1}{|Y_k|}\sum_{y\in Y_k}f_l(y)$ for all $k$
  \State $d_l\gets\W_p^p\big(\tfrac1N\sum_i\delta_{s_i},\,\tfrac1M\sum_k\delta_{t_k}\big)$ \Comment{sort + quantile matching, closed form \eqref{eq:1dot}}
\EndFor
\State \Return $\frac1L\sum_l d_l$
\end{algorithmic}}
\end{algorithm}

\subsection{Gradient flows: Practical implementation}
\label{sec:flows}

We optimize the empirical $\VSW_2^2$ objective by moving the support points of its inner measures, thereby obtaining a finite-dimensional particle-level gradient flow.  Throughout this calculation, the target and the slicing
law are fixed. Fix a target $\bbQ\in\PPO$ and define
$\mathcal{F}:\PPO\to\R$ by
$\mathcal{F}(\bbP)\;:=\;\VSW_2^2(\bbP,\bbQ)\;=\;\E_{f\sim\sigma}[\W_2^2\big((P_f)_\pf\bbP,(P_f)_\pf\bbQ\big)]$.

The discretization used in practice is as follows:
Let
$\bbQ=\frac1N\sum_{k=1}^N\delta_{\widehat\nu_k}$,
$\widehat\nu_k=\frac1m\sum_l\delta_{y_{kl}}$, and
$\bbP_X=\frac1N\sum_{i=1}^N\delta_{\widehat\mu_i}$,
$\widehat\mu_i=\frac1n\sum_j\delta_{x_{ij}}$,
with $X=(x_{ij})\in(\R^d)^{N\times n}$.
For simplicity we take equal outer sizes. The quantile formulation also
handles $N\ne M$. Define
\begin{equation}
\label{eq:objective}
\mathcal F(X)\;:=\;\VSW_2^2\big(\bbP_X,\bbQ\big)
=\E_{f\sim\sigma}\Big[\frac1N\sum_{i=1}^N\big(s_{(i)}(f)-t_{(i)}(f)\big)^2\Big],
\end{equation}
where $s_i(f)=\frac1n\sum_jf(x_{ij})$, $t_k(f)=\frac1m\sum_lf(y_{kl})$,
and $(\cdot)_{(i)}$ denotes increasing rearrangement. 
Here the projected laws
$\alpha_f=\frac1N\sum_i\delta_{s_i(f)}$ and
$\beta_f=\frac1N\sum_k\delta_{t_k(f)}$ are atomic, so transport is
determined by the sorting permutation %
rather than a continuous
monotone map. Lemma~\ref{lem:ties} ensures the required uniqueness.

\begin{proposition}%
\label{prop:gradient}
Assume Assumptions~\ref{ass:kernel}--\ref{ass:smooth}, with the smoothness
and moment conditions holding on a compact neighborhood of the particle
configurations. If the source clouds and the target clouds are each pairwise
distinct as measures, as in Lemma~\ref{lem:ties}, then $\mathcal F$ is differentiable at $X$, and
\begin{equation}
\label{eq:gradient}
\nabla_{x_{ij}}\mathcal F(X)\;=\;\frac{2}{Nn}\;\E_{f\sim\sigma}\Big[\big(s_i(f)-t_{\pi_f(i)}(f)\big)\,\nabla f(x_{ij})\Big].
\end{equation}
\end{proposition}
Consequently, the associated Euclidean particle gradient flow is
$\dot x_{ij}(t)=-\nabla_{x_{ij}}\mathcal F(X(t))$, with the
right-hand side given by \eqref{eq:gradient}.
The proof and corresponding algorithm are given in Appendix \ref{sec:proofs:vsw}. In the ideal setting, replacing the expectation in Proposition~\ref{prop:gradient} by an average over independent slices \(f_1,\ldots,f_L\sim\sigma\) gives a Monte Carlo estimate of the particle gradient. In practice, we represent the slices using finitely many random Fourier features, yielding the corresponding finite-feature approximation of this gradient. When the particles are produced by a differentiable generator, the particle gradients can be backpropagated through the generator by the chain rule.

\section{Experiments}
\label{sec:experiments}
We evaluate VSW across several tasks, demonstrating both performance and runtime advantages. For each experiment, we also optimize the baseline hyperparameters and implementations to ensure a fair comparison. Additional ablations and implementation details are provided in Appendices~\ref{app:ablation} and~\ref{app:impl}.

\noindent\textbf{Optimal Transport Dataset Distance.} Dataset distances provide a model-agnostic way to quantify similarity between
datasets, with applications such as transfer learning and dataset
selection~\citep{alvarez2020otdd,liu2025wasserstein}. Following previous work~\citep{sotdd, dsw}, we evaluate VSW as an efficient
approximation to the optimal transport dataset distance (OTDD), comparing
against s-OTDD~\citep{sotdd} and DSW~\citep{dsw}, using OTDD~\citep{alvarez2020otdd} as the reference. We
construct random pairs of sub-datasets from MNIST~\citep{lecun1998mnist}, FashionMNIST~\citep{xiao2017fashion}, and CIFAR-10~\citep{krizhevsky2009cifar10}
and represent each dataset through its class-conditional empirical
distributions. Following \citet{dsw}, we use identical dataset pairs and feature representations
across methods, remove the feature cost from OTDD so that it reduces to a
Wasserstein-over-Wasserstein distance between class-conditional distributions,
and use s-OTDD with its default implementation.

We report Pearson and Spearman correlations with OTDD while varying the number of slices. For DSW, the total projection budget is $S=PK$, where $P$ and $K$ denote the numbers of outer and inner projections, respectively. For each dataset, we use $30$ pairs of disjoint sub-datasets drawn
from the same training set. The two sub-datasets in each pair have equal size,
with that size varying from \(5{,}000\) to \(10{,}000\) across pairs.
Figure~\ref{fig:corr} shows that VSW achieves substantially
stronger correlation in the low-slice regime across all three datasets and
remains competitive at larger slice budgets under this protocol.
Figure~\ref{fig:otdd_runtime} compares runtime as the dataset size increases
for matched budgets of \(1\mathrm{k}\), \(5\mathrm{k}\), and
\(10\mathrm{k}\) slices. VSW is consistently the fastest method across all
three datasets and projection budgets under this resampling protocol.

\begin{figure*}[t!]
  \centering
  \includegraphics[width=.85\textwidth]{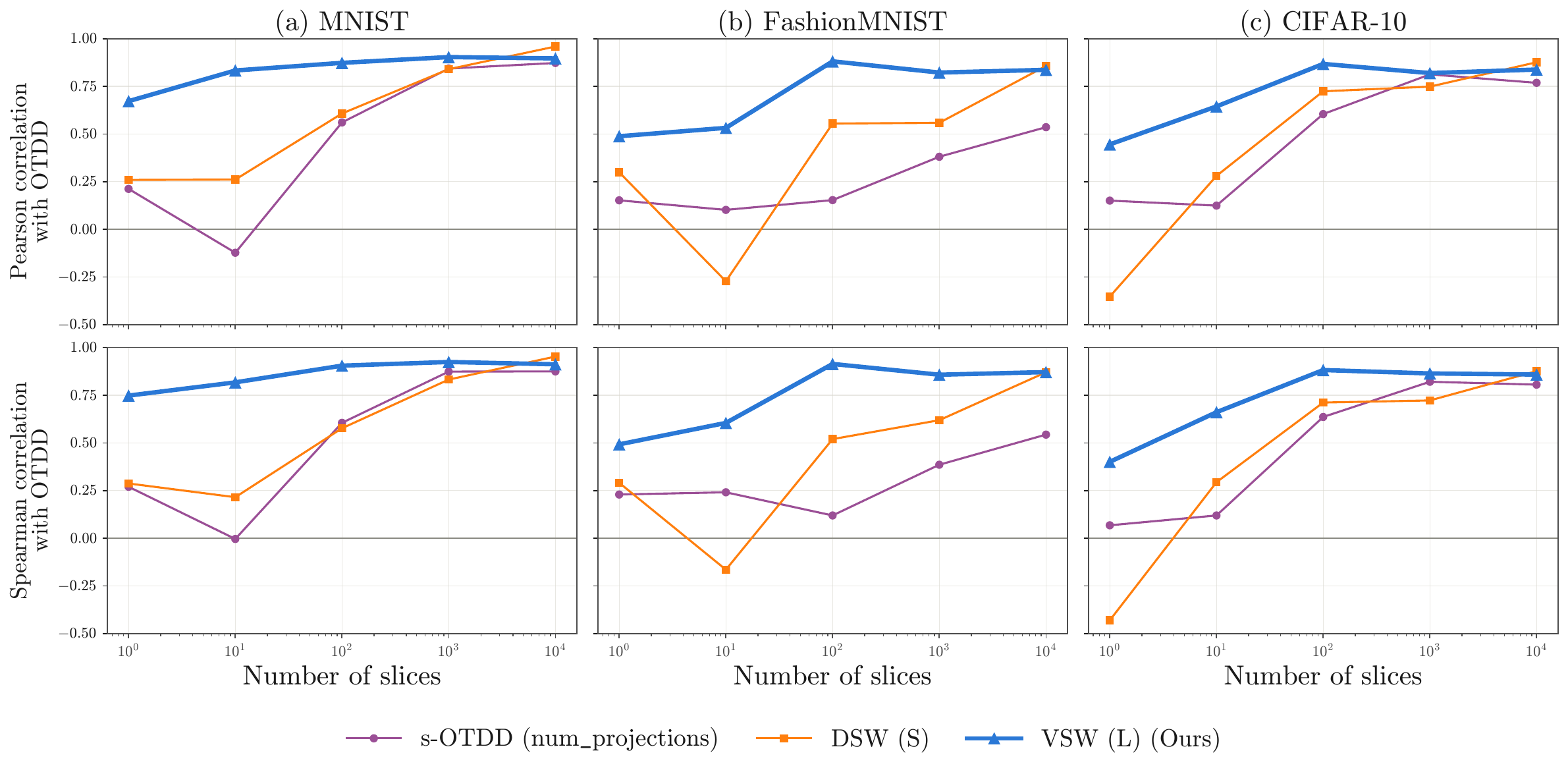}
  \caption{Pearson and Spearman correlations with exact OTDD versus the number of slices, computed over $30$ pairs per dataset. Both sub-datasets in a pair are drawn from the same training set and have equal size ($5{,}000$--$10{,}000$ samples); sizes vary across pairs.}
  \label{fig:corr}
\end{figure*}

\noindent\textbf{Comparing Distributions of Point Clouds.} Following the point-cloud distribution experiment of DSW~\citep{dsw}, we evaluate whether VSW captures changes in distributions of 3D point clouds from ModelNet-10~\citep{wu20153modelnet}. Starting from a fixed reference distribution, we independently vary (i) the number of shapes in the target distribution, (ii) the amount of Gaussian noise applied to the target points, and (iii) the number of points per target shape. We compare VSW with OT-NNA, exact WoW, and DSW. To facilitate direct comparison between methods with different scales, we report the normalized metrics together in Figure~\ref{fig:pc_toy}, with the reference configuration indicated by the dashed line.

VSW exhibits trends consistent with the OT-based baselines across all three perturbations. In particular, the distance decreases as the target contains more shapes, increases with stronger Gaussian perturbations, and remains comparatively stable once the point-cloud resolution is sufficiently high. These results indicate that VSW preserves the behavior of established distribution-comparison metrics while being substantially more efficient. We additionally report the wall-clock time in a realistic point cloud evaluation setting, comparing 662 ground-truth and 662 generated ShapeNet chairs~\citep{Chang2015ShapeNetAI} from PVD~\citep{zhou20213pvd}, each containing 2048 points. The results are shown in Table~\ref{tab:pc_time}. VSW requires only $0.0125$\,s per distance evaluation, compared with $0.0709$\,s for DSW, corresponding to approximately a $5.7\times$ reduction in runtime. OT-NNA and WoW are both substantially slower.

\begin{figure*}[t!]
  \centering
  \includegraphics[width=.9\textwidth]{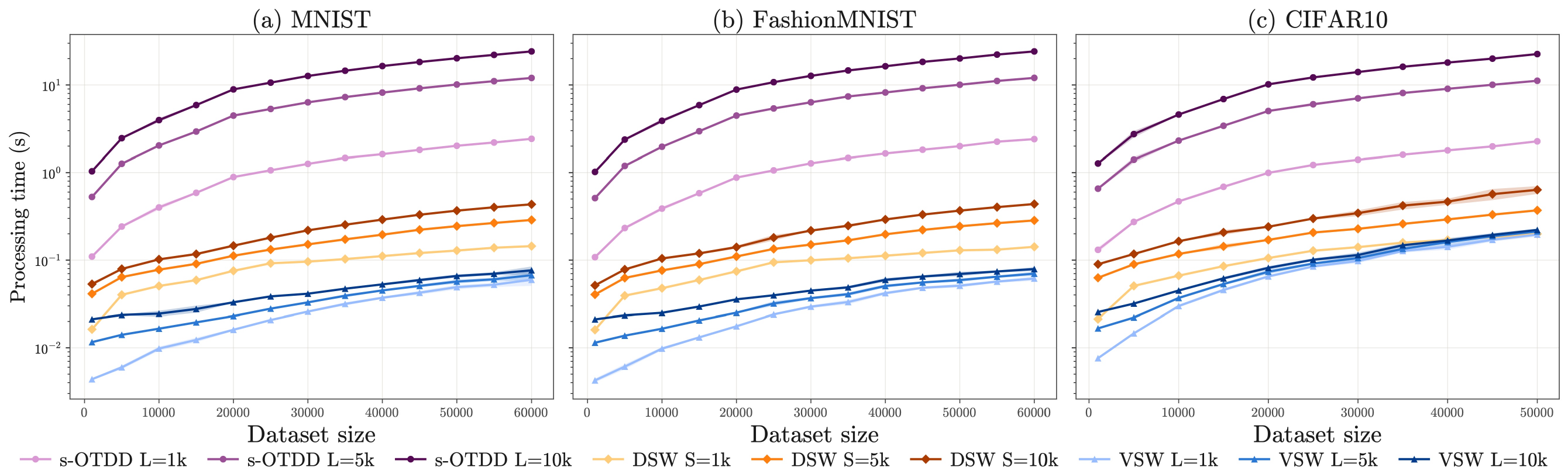}
  \caption{Runtime (s) vs the dataset size (each of 2 sets). VSW remains fastest among all methods.}
  \label{fig:otdd_runtime}  
\end{figure*}

\begin{table*}[t!]
\centering
\caption{Wall-clock time of a single distance evaluation between the 662
ground-truth ShapeNet chairs and 662 generated samples ($2048$ points per cloud). $^{\dagger}$ denotes a single run.
}
\label{tab:shapenet-chair-runtime}
\label{tab:pc_time}
{\begin{tabular}{lcccc}
\toprule
 & OT-NNA & WoW & DSW & VSW (Ours) \\
\midrule
runtime (s) & $6077.65^{\dagger}$ & $3094.22^{\dagger}$ & $0.0709 \pm 0.00017$ & $\mathbf{0.0125 \pm 0.00003}$ \\
\bottomrule
\end{tabular}}
\end{table*}

\begin{figure*}[t!]
  \centering
  \includegraphics[width=.9\textwidth]{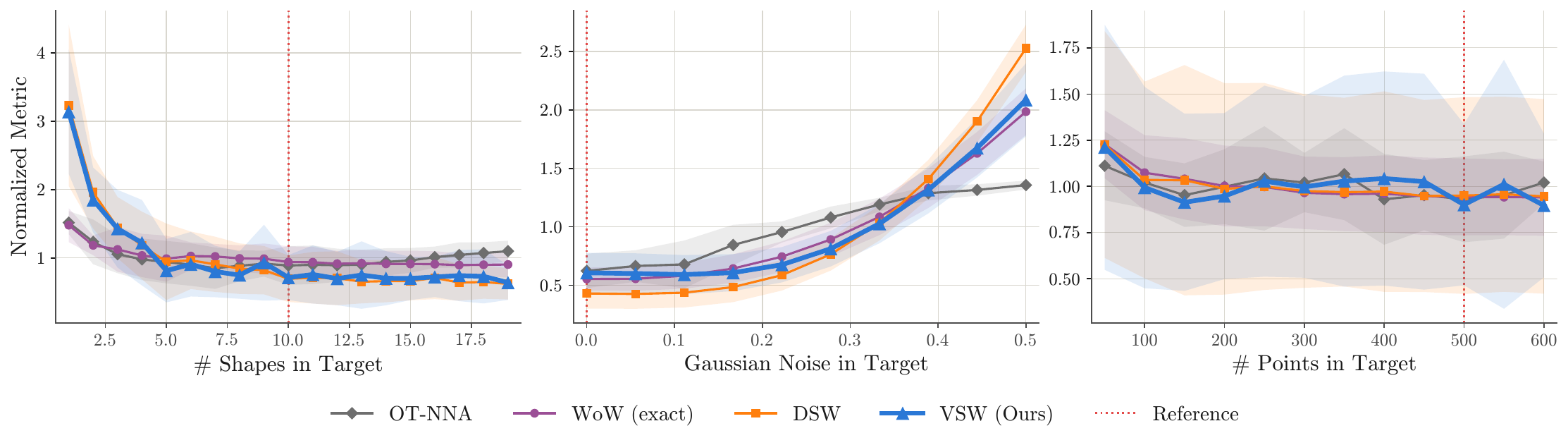}
  \caption{Point cloud comparison metric trends.}
  \label{fig:pc_toy}
\end{figure*}

\begin{figure*}[t!]
  \centering
  \includegraphics[width=.75\textwidth]{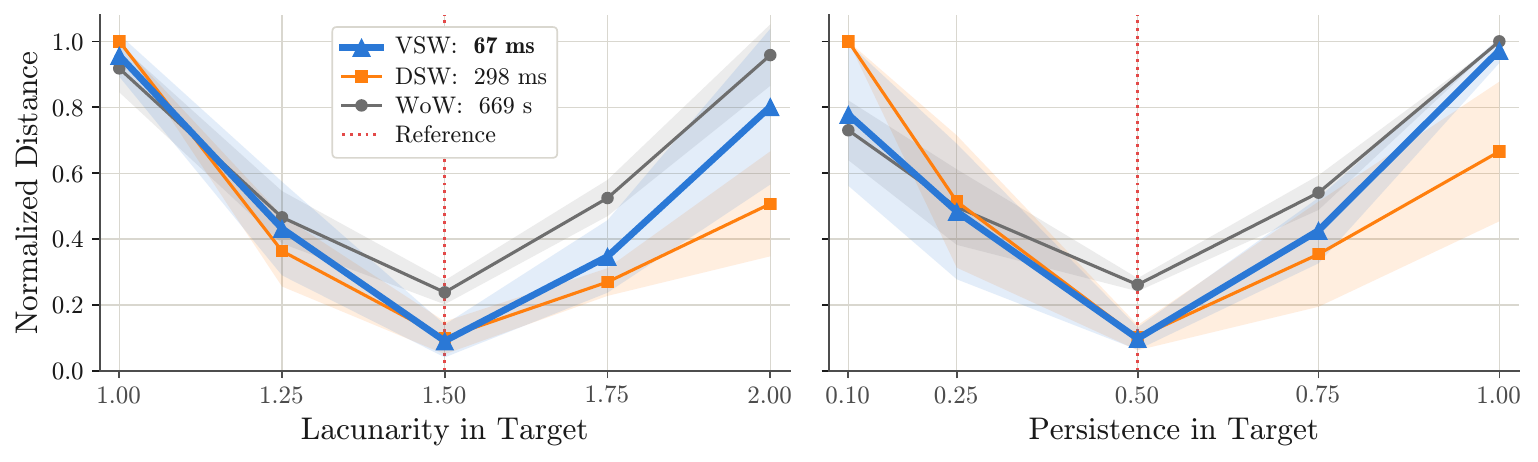}
  \caption{Comparing image distributions via patch distributions. VSW remains the fastest method, while achieving the same trend as baselines.}
  \label{fig:patch}
\end{figure*}

\noindent\textbf{Comparing Image Distributions via Patch Distributions.} Following DSW~\citep{dsw}, we evaluate VSW for comparing image distributions through their patch distributions. We use synthetic Perlin-noise textures and independently vary the \emph{lacunarity} and \emph{persistence} of the target distribution relative to a fixed reference. We compare VSW against DSW and exact Wasserstein-over-Wasserstein (WoW), and normalize each metric for visualization.

As shown in Figure ~\ref{fig:patch}, all three methods correctly attain their minimum near the reference parameters, while increasing as the target distribution deviates from the reference. VSW closely follows the trends of DSW and WoW, indicating that it preserves meaningful sensitivity to changes in image-patch distributions. At the same time, VSW requires only $67$ ms per comparison, compared with $298$ ms for DSW and $669$ s for exact WoW, making it approximately $4.4\times$ faster than DSW and nearly four orders of magnitude faster than exact WoW.

\noindent\textbf{Gradient Flow on Natural Images.}
We use ImageNet~\citep{imagenet} to evaluate the VSW particle gradient on a high-dimensional dataset. We optimize five clouds of twenty noise-initialized
particles against a fixed target containing twenty training images from
each of five classes. Figure~\ref{fig:grad_flow} shows the particle
trajectories in pixel space and in a two-dimensional PCA projection of
CLIP ViT-L/14 embeddings~\citep{radford2021learning}, fitted to the target
images. The particles progressively resemble target images, with projected
trajectories moving toward the target clusters. Nearest-neighbor
ground-truth images provide a qualitative reference for the final
particles. These visualizations illustrate optimization toward the fixed
empirical target.

\begin{figure*}[t!]
  \centering
  \includegraphics[width=\textwidth]{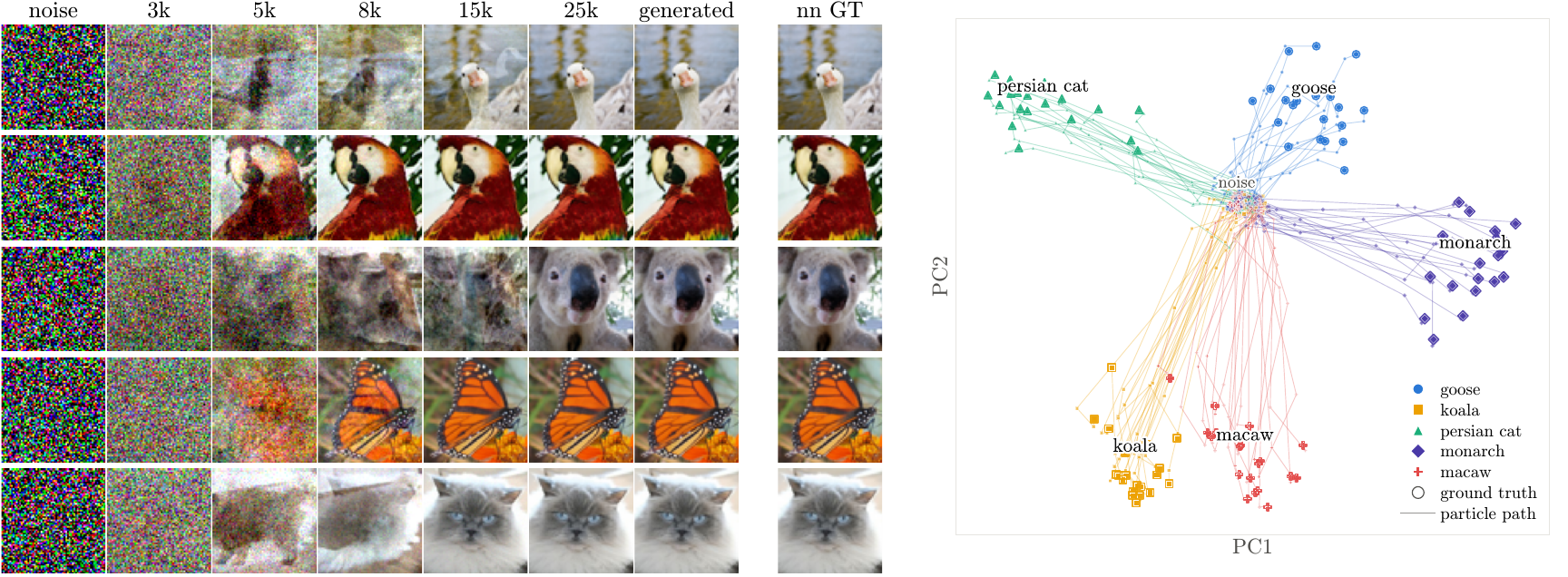}
  \caption{Particle optimization toward fixed ImageNet targets in pixel space \textbf{(Left)}, and in a 2D PCA projection of CLIP ViT-L/14 embeddings \textbf{(Right)}.}
  \label{fig:grad_flow}
\end{figure*}

\noindent\textbf{Parametric Generative Modeling.} We next evaluate VSW in a parametric generative modeling setting on two
2D hierarchical datasets: a checkerboard distribution and an
eight-component Gaussian mixture. Each dataset is represented as a
distribution of inner distributions, with an outer latent variable selecting
the component and inner noise generating samples within it. We train a
lightweight MLP using the corresponding objective. We compare VSW with DSW~\citep{dsw} and MMDSW~\citep{bonet2025flowing}. As shown in
Figure~\ref{fig:cb_gen}, VSW and MMDSW recover the main structure of both
targets, and VSW closely matches the held-out distributions. VSW also
reduces the Wasserstein error substantially faster in wall-clock time,
reaching a low error earlier than DSW and MMDSW on both datasets.

\begin{figure*}[t!]
  \centering
  \includegraphics[width=\textwidth]{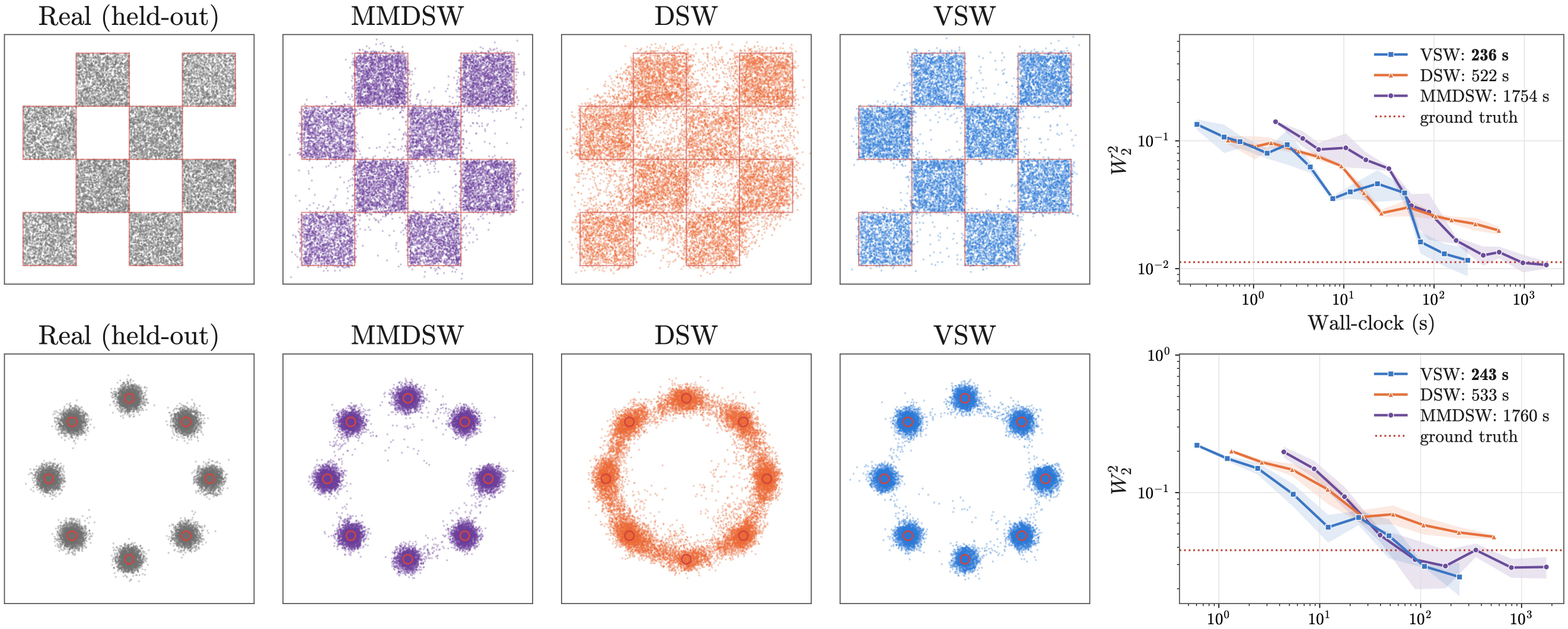}
  \caption{Parametric generation on Checkerboard and 8-Gaussians datasets.}
  \label{fig:cb_gen}
\end{figure*}

\noindent\textbf{Shape Classification.}
Following the protocol of DSW~\citep{dsw}, we evaluate VSW on
2D Shapes~\citep{beier2022gromov}, Animals~\citep{sumner2004meshes}, FAUST-500/1000\newline~\citep{faust} ($500$ and $1000$ vertices per shape), and MNIST-2000~\citep{dsw} using local distance distributions. For each distance, we precompute the pairwise distance matrix and perform 3-NN classification over 1000 random 25\%/75\% train/test splits, reporting mean accuracy and CPU time per distance. VSW, SQW, and STLB share one random projection basis across the distance matrix, while AE caches its per-shape energy self-terms. TLB and GW are deterministic and do not include such amortization.

Table~\ref{tab:shape_classification} shows competitive mean accuracies across these datasets, with VSW consistently faster than the other sliced baselines. Compared with SQW, VSW is approximately
$2.3$--$2.8\times$ faster across all datasets. The advantage is especially clear for larger shapes, requiring only $3.76$ ms on FAUST-1000 and
$7.48$ ms on MNIST-2000. VSW provides a favorable accuracy--runtime trade-off for this task.

\begin{table*}[t]
    \centering
    \caption{Shape classification with KNN (Acc., $\uparrow$) and runtime per distance computation (Time).}
    \label{tab:shape_classification}
    \scriptsize
    \setlength{\tabcolsep}{4.5pt}
    \renewcommand{\arraystretch}{1.08}
    \resizebox{\textwidth}{!}{%
    \begin{tabular}{l cc cc cc cc cc}
        \toprule
        \multirow{2}{*}{Distance}
        & \multicolumn{2}{c}{2D Shapes}
        & \multicolumn{2}{c}{Animals}
        & \multicolumn{2}{c}{FAUST-500}
        & \multicolumn{2}{c}{FAUST-1000}
        & \multicolumn{2}{c}{MNIST-2000} \\
        \cmidrule(lr){2-3}
        \cmidrule(lr){4-5}
        \cmidrule(lr){6-7}
        \cmidrule(lr){8-9}
        \cmidrule(lr){10-11}
        & Acc. (\%) & Time (ms)
        & Acc. (\%) & Time (ms)
        & Acc. (\%) & Time (ms)
        & Acc. (\%) & Time (ms)
        & Acc. (\%) & Time (ms) \\
        \midrule

        VSW
        & $97.7 \pm 2.6$ & $0.35$
        & $99.0 \pm 1.3$ & $0.34$
        & $39.0 \pm 5.5$ & $2.80$
        & $41.2 \pm 6.0$ & $3.76$
        & $96.5 \pm 3.4$ & $7.48$ \\

        \midrule

        SQW
        & $99.4 \pm 1.2$ & $0.93$
        & $99.3 \pm 1.0$ & $0.95$
        & $37.5 \pm 5.5$ & $6.51$
        & $40.8 \pm 6.0$ & $10.15$
        & $86.0 \pm 4.3$ & $18.46$ \\

        TLB
        & $100.0 \pm 0.2$ & $0.91$
        & $100.0 \pm 0.0$ & $0.97$
        & $36.7 \pm 5.6$ & $40.96$
        & $40.2 \pm 6.0$ & $163.06$
        & $95.5 \pm 3.3$ & $653.84$ \\

        STLB
        & $98.5 \pm 2.0$ & $0.99$
        & $99.0 \pm 1.3$ & $1.00$
        & $35.2 \pm 5.5$ & $10.01$
        & $38.4 \pm 5.9$ & $10.65$
        & $82.2 \pm 3.8$ & $18.36$ \\

        AE
        & $99.7 \pm 0.8$ & $0.13$
        & $97.9 \pm 1.8$ & $0.13$
        & $37.5 \pm 5.7$ & $2.63$
        & $42.0 \pm 5.9$ & $4.46$
        & $93.2 \pm 4.0$ & $19.79$ \\

        GW
        & $99.7 \pm 0.6$ & $3.1$
        & $100.0 \pm 0.0$ & $6.2$
        & $29.2 \pm 4.4$ & $346.3$
        & $33.0 \pm 5.3$ & $1824.7$
        & \multicolumn{2}{c}{\textemdash{} timeout \textemdash{}} \\

        \bottomrule
        
    \end{tabular}%
    }
\end{table*}

\section{Conclusion and Future work}

We introduced Vertical Sliced Wasserstein (VSW), a scalable distance for meta-measures based on slicing through continuous test-function projections rather than the transport geometry of the inner measures. We showed that VSW is a valid metric that metrizes weak convergence, and that $\mathrm{VSW}_2$ admits a Hilbertian embedding, yielding characteristic kernels on meta-measures. We further established dimension-independent two-stage convergence rates and derived a practical particle gradient for optimization. Across dataset comparison, shape classification, point-cloud and image-distribution comparison, gradient flow and parametric generation, VSW provides competitive agreement with transport-based distances while substantially reducing computation.

Our analysis also leaves several limitations and open directions. Promising directions include extending parametric generative modeling beyond low-dimensional examples to higher-dimensional and structured domains, such as point cloud and image generation. For instance, a continuum WoW-flow interpretation (where, based on the differential geometry of
\citet{bonet2025flowing},  VSW defines the objective,
 while $(\mathcal{P}(\PO),\WoW_2)$ defines the geometry of the descent) together with a convergence of the discrete optimization schemes requires a separate theoretical analysis.

\section*{Acknowledgments}

This work was supported in part by NSF CAREER Award No. 233989, NSF DMS Award No. 2603773, and a generous gift from Toyota Infotech Lab. The authors also gratefully acknowledge the computing resources provided by the NVIDIA Academic Grant Program.

\bibliographystyle{plainnat}
\bibliography{vsw_arxiv}

\begin{thebibliography}{49}
\providecommand{\natexlab}[1]{#1}
\providecommand{\url}[1]{\texttt{#1}}
\expandafter\ifx\csname urlstyle\endcsname\relax
  \providecommand{\doi}[1]{doi: #1}\else
  \providecommand{\doi}{doi: \begingroup \urlstyle{rm}\Url}\fi

\bibitem[Adler and Taylor(2007)]{adler2007random}
Robert~J Adler and Jonathan~E Taylor.
\newblock \emph{Random fields and geometry}.
\newblock Springer, 2007.

\bibitem[Alvarez-Melis and Fusi(2020)]{alvarez2020otdd}
David Alvarez-Melis and Nicolo Fusi.
\newblock Geometric dataset distances via optimal transport.
\newblock \emph{Advances in Neural Information Processing Systems},
  33:\penalty0 21428--21439, 2020.

\bibitem[Beier et~al.(2022)Beier, Beinert, and Steidl]{beier2022gromov}
Florian Beier, Robert Beinert, and Gabriele Steidl.
\newblock On a linear {G}romov–{W}asserstein distance.
\newblock \emph{IEEE Transactions on Image Processing}, 31:\penalty0
  7292--7305, 2022.
\newblock \doi{10.1109/TIP.2022.3221286}.

\bibitem[Berg et~al.(1984)Berg, Christensen, and Ressel]{berg1984harmonic}
Christian Berg, Jens Peter~Reus Christensen, and Paul Ressel.
\newblock \emph{Harmonic analysis on semigroups: theory of positive definite
  and related functions}, volume 100.
\newblock Springer, 1984.

\bibitem[Bobkov and Ledoux(2019)]{bobkov2019one}
Sergey~G. Bobkov and Michel Ledoux.
\newblock One-dimensional empirical measures, order statistics, and
  {K}antorovich transport distances.
\newblock \emph{Memoirs of the American Mathematical Society}, 2019.
\newblock URL \url{https://api.semanticscholar.org/CorpusID:209912842}.

\bibitem[Bogachev(1998)]{bogachev1998gaussian}
Vladimir~I. Bogachev.
\newblock \emph{{G}aussian Measures}.
\newblock American Mathematical Society, 1998.

\bibitem[Bogo et~al.(2014)Bogo, Romero, Loper, and Black]{faust}
Federica Bogo, Javier Romero, Matthew Loper, and Michael~J. Black.
\newblock {FAUST}: Dataset and evaluation for {3D} mesh registration.
\newblock In \emph{2014 IEEE Conference on Computer Vision and Pattern
  Recognition}, pages 3794--3801, 2014.
\newblock \doi{10.1109/CVPR.2014.491}.

\bibitem[Bonet et~al.(2025{\natexlab{a}})Bonet, Cazelles, Drumetz, and
  Courty]{bonet2025busemann}
Cl{\'e}ment Bonet, Elsa Cazelles, Lucas Drumetz, and Nicolas Courty.
\newblock Busemann functions in the wasserstein space: existence, closed-forms,
  and applications to slicing.
\newblock \emph{arXiv preprint arXiv:2510.04579}, 2025{\natexlab{a}}.

\bibitem[Bonet et~al.(2025{\natexlab{b}})Bonet, Vauthier, and
  Korba]{bonet2025flowing}
Cl{\'e}ment Bonet, Christophe Vauthier, and Anna Korba.
\newblock Flowing datasets with wasserstein over wasserstein gradient flows.
\newblock In \emph{Forty-second International Conference on Machine Learning},
  2025{\natexlab{b}}.
\newblock URL \url{https://openreview.net/forum?id=I1OHPb4zWo}.

\bibitem[Catalano and Lavenant(2024)]{catalano2024hierarchical}
Marta Catalano and Hugo Lavenant.
\newblock Hierarchical integral probability metrics: A distance on random
  probability measures with low sample complexity.
\newblock In \emph{Forty-first International Conference on Machine Learning},
  2024.
\newblock URL \url{https://openreview.net/forum?id=cmy38XZlJu}.

\bibitem[Chang et~al.(2015)Chang, Funkhouser, Guibas, Hanrahan, Huang, Li,
  Savarese, Savva, Song, Su, Xiao, Yi, and Yu]{Chang2015ShapeNetAI}
Angel~X. Chang, Thomas~A. Funkhouser, Leonidas~J. Guibas, Pat Hanrahan, Qi-Xing
  Huang, Zimo Li, Silvio Savarese, Manolis Savva, Shuran Song, Hao Su,
  Jianxiong Xiao, L.~Yi, and Fisher Yu.
\newblock {ShapeNet}: An information-rich {3D} model repository.
\newblock \emph{ArXiv}, abs/1512.03012, 2015.
\newblock URL \url{https://api.semanticscholar.org/CorpusID:2554264}.

\bibitem[Chizat et~al.(2018)Chizat, Peyr{\'e}, Schmitzer, and
  Vialard]{chizat2018interpolating}
L{\'e}na{\"i}c Chizat, Gabriel Peyr{\'e}, Bernhard Schmitzer, and
  Fran{\c{c}}ois-Xavier Vialard.
\newblock An interpolating distance between optimal transport and {Fisher--Rao}
  metrics.
\newblock \emph{Foundations of Computational Mathematics}, 18\penalty0
  (1):\penalty0 1--44, 2018.
\newblock \doi{10.1007/s10208-016-9331-y}.
\newblock URL \url{https://doi.org/10.1007/s10208-016-9331-y}.

\bibitem[Christmann and Steinwart(2010)]{christmann2010universal}
Andreas Christmann and Ingo Steinwart.
\newblock Universal kernels on non-standard input spaces.
\newblock In J.~Lafferty, C.~Williams, J.~Shawe-Taylor, R.~Zemel, and
  A.~Culotta, editors, \emph{Advances in Neural Information Processing
  Systems}, volume~23. Curran Associates, Inc., 2010.
\newblock URL
  \url{https://proceedings.neurips.cc/paper_files/paper/2010/file/4e0cb6fb5fb446d1c92ede2ed8780188-Paper.pdf}.

\bibitem[Deng et~al.(2009)Deng, Dong, Socher, Li, Li, and Fei-Fei]{imagenet}
Jia Deng, Wei Dong, Richard Socher, Li-Jia Li, Kai Li, and Li~Fei-Fei.
\newblock {ImageNet}: A large-scale hierarchical image database.
\newblock In \emph{2009 IEEE Conference on Computer Vision and Pattern
  Recognition}, pages 248--255, 2009.
\newblock \doi{10.1109/CVPR.2009.5206848}.

\bibitem[Dudley(1967)]{dudley1967sizes}
Richard~M Dudley.
\newblock The sizes of compact subsets of {H}ilbert space and continuity of
  {G}aussian processes.
\newblock \emph{Journal of Functional Analysis}, 1\penalty0 (3):\penalty0
  290--330, 1967.

\bibitem[Dukler et~al.(2019)Dukler, Li, Lin, and
  Montufar]{dukler2019wasserstein}
Yonatan Dukler, Wuchen Li, Alex Lin, and Guido Montufar.
\newblock {W}asserstein of {W}asserstein loss for learning generative models.
\newblock In Kamalika Chaudhuri and Ruslan Salakhutdinov, editors,
  \emph{Proceedings of the 36th International Conference on Machine Learning},
  volume~97 of \emph{Proceedings of Machine Learning Research}, pages
  1716--1725. PMLR, 09--15 Jun 2019.
\newblock URL \url{https://proceedings.mlr.press/v97/dukler19a.html}.

\bibitem[Fernique(1970)]{fernique1970integrabilite}
Xavier Fernique.
\newblock Int{\'e}grabilit{\'e} des vecteurs gaussiens.
\newblock \emph{CR Acad. Sci. Paris S{\'e}r. AB}, 270:\penalty0 A1698--A1699,
  1970.

\bibitem[Flamary et~al.(2021)Flamary, Courty, Gramfort, Alaya, Boisbunon,
  Chambon, Chapel, Corenflos, Fatras, Fournier, Gautheron, Gayraud, Janati,
  Rakotomamonjy, Redko, Rolet, Schutz, Seguy, Sutherland, Tavenard, Tong, and
  Vayer]{flamary2021pot}
R{\'e}mi Flamary, Nicolas Courty, Alexandre Gramfort, Mokhtar~Z. Alaya,
  Aur{\'e}lie Boisbunon, Stanislas Chambon, Laetitia Chapel, Adrien Corenflos,
  Kilian Fatras, Nemo Fournier, L{\'e}o Gautheron, Nathalie~T.H. Gayraud,
  Hicham Janati, Alain Rakotomamonjy, Ievgen Redko, Antoine Rolet, Antony
  Schutz, Vivien Seguy, Danica~J. Sutherland, Romain Tavenard, Alexander Tong,
  and Titouan Vayer.
\newblock {POT}: {Python} optimal transport.
\newblock \emph{J. Mach. Learn. Res.}, 22\penalty0 (1), January 2021.
\newblock ISSN 1532-4435.

\bibitem[Fournier and Guillin(2015)]{fournier2015rate}
Nicolas Fournier and Arnaud Guillin.
\newblock On the rate of convergence in {W}asserstein distance of the empirical
  measure.
\newblock \emph{Probability theory and related fields}, 162\penalty0
  (3):\penalty0 707--738, 2015.

\bibitem[Gretton et~al.(2012)Gretton, Borgwardt, Rasch, Sch{{\"o}}lkopf, and
  Smola]{gretton2012kernel}
Arthur Gretton, Karsten~M. Borgwardt, Malte~J. Rasch, Bernhard Sch{{\"o}}lkopf,
  and Alexander Smola.
\newblock A kernel two-sample test.
\newblock \emph{Journal of Machine Learning Research}, 13\penalty0
  (25):\penalty0 723--773, 2012.
\newblock URL \url{http://jmlr.org/papers/v13/gretton12a.html}.

\bibitem[Han(2023)]{han2023sliced}
Ruiyu Han.
\newblock Sliced wasserstein distance between probability measures on hilbert
  spaces.
\newblock \emph{arXiv preprint arXiv:2307.05802}, 2023.

\bibitem[Haviv et~al.(2025)Haviv, Pooladian, Pe'er, and
  Amos]{haviv2024wasserstein}
Doron Haviv, Aram-Alexandre Pooladian, Dana Pe'er, and Brandon Amos.
\newblock Wasserstein flow matching: Generative modeling over families of
  distributions, 2025.
\newblock URL \url{https://openreview.net/forum?id=HB4lr0ykTi}.

\bibitem[Kallenberg(2017)]{kallenberg2017random}
Olav Kallenberg.
\newblock \emph{Random Measures, Theory and Applications}.
\newblock Springer, 2017.

\bibitem[Krizhevsky and Hinton(2009)]{krizhevsky2009cifar10}
Alex Krizhevsky and Geoffrey Hinton.
\newblock Learning multiple layers of features from tiny images.
\newblock Technical report, University of Toronto, 2009.

\bibitem[LeCun et~al.(1998)LeCun, Bottou, Bengio, and Haffner]{lecun1998mnist}
Yann LeCun, L{\'e}on Bottou, Yoshua Bengio, and Patrick Haffner.
\newblock Gradient-based learning applied to document recognition.
\newblock \emph{Proceedings of the IEEE}, 86\penalty0 (11):\penalty0
  2278--2324, 1998.

\bibitem[Ledoux and Talagrand(1991)]{ledoux1991probability}
Michel Ledoux and Michel Talagrand.
\newblock \emph{Probability in {B}anach Spaces: isoperimetry and processes},
  volume~23.
\newblock Springer, 1991.

\bibitem[Liero et~al.(2018)Liero, Mielke, and Savar{\'e}]{liero2018optimal}
Matthias Liero, Alexander Mielke, and Giuseppe Savar{\'e}.
\newblock Optimal entropy-transport problems and a new {Hellinger--Kantorovich}
  distance between positive measures.
\newblock \emph{Inventiones mathematicae}, 211\penalty0 (3):\penalty0
  969--1117, 2018.
\newblock \doi{10.1007/s00222-017-0759-8}.
\newblock URL \url{https://doi.org/10.1007/s00222-017-0759-8}.

\bibitem[Liu et~al.(2025)Liu, Bai, Lu, Soltoggio, and
  Kolouri]{liu2025wasserstein}
Xinran Liu, Yikun Bai, Yuzhe Lu, Andrea Soltoggio, and Soheil Kolouri.
\newblock {W}asserstein task embedding for measuring task similarities.
\newblock \emph{Neural Networks}, 181:\penalty0 106796, 2025.

\bibitem[Micchelli et~al.(2006)Micchelli, Xu, and
  Zhang]{micchelli2006universal}
Charles~A. Micchelli, Yuesheng Xu, and Haizhang Zhang.
\newblock Universal kernels.
\newblock \emph{Journal of Machine Learning Research}, 7\penalty0
  (95):\penalty0 2651--2667, 2006.
\newblock URL \url{http://jmlr.org/papers/v7/micchelli06a.html}.

\bibitem[Nadjahi et~al.(2020)Nadjahi, Durmus, Chizat, Kolouri, Shahrampour, and
  Simsekli]{nadjahi2020statistical}
Kimia Nadjahi, Alain Durmus, L\'{e}na\"{\i}c Chizat, Soheil Kolouri, Shahin
  Shahrampour, and Umut Simsekli.
\newblock Statistical and topological properties of sliced probability
  divergences.
\newblock In H.~Larochelle, M.~Ranzato, R.~Hadsell, M.F. Balcan, and H.~Lin,
  editors, \emph{Advances in Neural Information Processing Systems}, volume~33,
  pages 20802--20812. Curran Associates, Inc., 2020.
\newblock URL
  \url{https://proceedings.neurips.cc/paper_files/paper/2020/file/eefc9e10ebdc4a2333b42b2dbb8f27b6-Paper.pdf}.

\bibitem[Nadjahi et~al.(2021)Nadjahi, Durmus, Jacob, Badeau, and
  Simsekli]{nadjahi2021fast}
Kimia Nadjahi, Alain Durmus, Pierre~E Jacob, Roland Badeau, and Umut Simsekli.
\newblock Fast approximation of the sliced-{W}asserstein distance using
  concentration of random projections.
\newblock \emph{Advances in Neural Information Processing Systems},
  34:\penalty0 12411--12424, 2021.

\bibitem[Naslidnyk et~al.(2025)Naslidnyk, Chau, Briol, and
  Muandet]{naslidnyk2025kernel}
Masha Naslidnyk, Siu~Lun Chau, Francois-Xavier Briol, and Krikamol Muandet.
\newblock Kernel quantile embeddings and associated probability metrics.
\newblock In \emph{Proceedings of the 42nd International Conference on Machine
  Learning}, volume 267 of \emph{Proceedings of Machine Learning Research},
  pages 45770--45800. PMLR, 2025.
\newblock URL \url{https://proceedings.mlr.press/v267/naslidnyk25a.html}.

\bibitem[Nguyen and Mueller(2026)]{nguyen2024summarizing}
Khai Nguyen and Peter Mueller.
\newblock Summarizing nonparametric bayesian mixture posteriors—sliced
  optimal transport metrics for gaussian mixtures.
\newblock \emph{Journal of Computational and Graphical Statistics}, 0\penalty0
  (0):\penalty0 1--12, 2026.
\newblock \doi{10.1080/10618600.2026.2634831}.

\bibitem[Nguyen et~al.(2025)Nguyen, Nguyen, Pham, and Ho]{sotdd}
Khai Nguyen, Hai Nguyen, Tuan Pham, and Nhat Ho.
\newblock Lightspeed geometric dataset distance via sliced optimal transport.
\newblock In \emph{Proceedings of the 42nd International Conference on Machine
  Learning}, ICML'25. JMLR.org, 2025.

\bibitem[Piening and Beinert(2026)]{dsw}
Moritz Piening and Robert Beinert.
\newblock Slicing {W}asserstein over {W}asserstein via functional optimal
  transport.
\newblock In \emph{International Conference on Learning Representations},
  volume 2026, pages 84708--84740, 2026.
\newblock URL
  \url{https://proceedings.iclr.cc/paper_files/paper/2026/hash/88549904b871acbd676fb7c81581b0c2-Abstract-Conference.html}.

\bibitem[Radford et~al.(2021)Radford, Kim, Hallacy, Ramesh, Goh, Agarwal,
  Sastry, Askell, Mishkin, Clark, et~al.]{radford2021learning}
Alec Radford, Jong~Wook Kim, Chris Hallacy, Aditya Ramesh, Gabriel Goh,
  Sandhini Agarwal, Girish Sastry, Amanda Askell, Pamela Mishkin, Jack Clark,
  et~al.
\newblock Learning transferable visual models from natural language
  supervision.
\newblock In \emph{International conference on machine learning}, pages
  8748--8763. PmLR, 2021.

\bibitem[Rahimi and Recht(2007)]{rahimi2007random}
Ali Rahimi and Benjamin Recht.
\newblock Random features for large-scale kernel machines.
\newblock \emph{Advances in neural information processing systems}, 20, 2007.

\bibitem[Rustamov and Majumdar(2023)]{rustamov2023intrinsic}
Raif~M. Rustamov and Subhabrata Majumdar.
\newblock Intrinsic sliced {W}asserstein distances for comparing collections of
  probability distributions on manifolds and graphs.
\newblock In \emph{Proceedings of the 40th International Conference on Machine
  Learning}, volume 202 of \emph{Proceedings of Machine Learning Research},
  pages 29388--29415. PMLR, 2023.
\newblock URL \url{https://proceedings.mlr.press/v202/rustamov23a.html}.

\bibitem[Schoenberg(1938)]{schoenberg1938metric}
Isaac~J Schoenberg.
\newblock Metric spaces and positive definite functions.
\newblock \emph{Transactions of the American Mathematical Society}, 44\penalty0
  (3):\penalty0 522--536, 1938.

\bibitem[Sriperumbudur et~al.(2011)Sriperumbudur, Fukumizu, and
  Lanckriet]{sriperumbudur2011universality}
Bharath~K. Sriperumbudur, Kenji Fukumizu, and Gert~R.G. Lanckriet.
\newblock Universality, characteristic kernels and rkhs embedding of measures.
\newblock \emph{Journal of Machine Learning Research}, 12\penalty0
  (70):\penalty0 2389--2410, 2011.
\newblock URL \url{http://jmlr.org/papers/v12/sriperumbudur11a.html}.

\bibitem[Steinwart(2001)]{steinwart2001influence}
Ingo Steinwart.
\newblock On the influence of the kernel on the consistency of support vector
  machines.
\newblock \emph{Journal of machine learning research}, 2\penalty0
  (Nov):\penalty0 67--93, 2001.

\bibitem[Sumner and Popovi{\'c}(2004)]{sumner2004meshes}
Robert~W. Sumner and Jovan Popovi{\'c}.
\newblock Deformation transfer for triangle meshes.
\newblock \emph{ACM Trans. Graph.}, 23\penalty0 (3):\penalty0 399–405, August
  2004.
\newblock ISSN 0730-0301.
\newblock \doi{10.1145/1015706.1015736}.
\newblock URL \url{https://doi.org/10.1145/1015706.1015736}.

\bibitem[Szab{{\'o}} et~al.(2016)Szab{{\'o}}, Sriperumbudur, P{{\'o}}czos, and
  Gretton]{szabo2016learning}
Zolt{{\'a}}n Szab{{\'o}}, Bharath~K. Sriperumbudur, Barnab{{\'a}}s
  P{{\'o}}czos, and Arthur Gretton.
\newblock Learning theory for distribution regression.
\newblock \emph{Journal of Machine Learning Research}, 17\penalty0
  (152):\penalty0 1--40, 2016.
\newblock URL \url{http://jmlr.org/papers/v17/14-510.html}.

\bibitem[van~der Vaart and van Zanten(2008)]{vaart2008reproducing}
Aad~W. van~der Vaart and J.~Harry van Zanten.
\newblock Reproducing kernel {H}ilbert spaces of {G}aussian priors.
\newblock In \emph{Pushing the Limits of Contemporary Statistics: Contributions
  in Honor of Jayanta K. Ghosh}, pages 200--222. Institute of Mathematical
  Statistics, 2008.

\bibitem[Wu et~al.(2015)Wu, Song, Khosla, Yu, Zhang, Tang, and
  Xiao]{wu20153modelnet}
Zhirong Wu, Shuran Song, Aditya Khosla, Fisher Yu, Linguang Zhang, Xiaoou Tang,
  and Jianxiong Xiao.
\newblock {3D} {ShapeNets}: A deep representation for volumetric shapes.
\newblock In \emph{Proceedings of the IEEE conference on computer vision and
  pattern recognition}, pages 1912--1920, 2015.

\bibitem[Xiao et~al.(2017)Xiao, Rasul, and Vollgraf]{xiao2017fashion}
Han Xiao, Kashif Rasul, and Roland Vollgraf.
\newblock {Fashion-MNIST}: a novel image dataset for benchmarking machine
  learning algorithms.
\newblock \emph{arXiv preprint arXiv:1708.07747}, 2017.

\bibitem[Yang et~al.(2019)Yang, Huang, Hao, Liu, Belongie, and
  Hariharan]{yang2019pointflow}
Guandao Yang, Xun Huang, Zekun Hao, Ming-Yu Liu, Serge Belongie, and Bharath
  Hariharan.
\newblock {PointFlow}: {3D} point cloud generation with continuous normalizing
  flows.
\newblock In \emph{2019 IEEE/CVF International Conference on Computer Vision
  (ICCV)}, pages 4540--4549, 2019.
\newblock \doi{10.1109/ICCV.2019.00464}.

\bibitem[Yurochkin et~al.(2019)Yurochkin, Claici, Chien, Mirzazadeh, and
  Solomon]{yurochkin2019hierarchical}
Mikhail Yurochkin, Sebastian Claici, Edward Chien, Farzaneh Mirzazadeh, and
  Justin Solomon.
\newblock Hierarchical optimal transport for document representation.
\newblock In \emph{Proceedings of the 33rd International Conference on Neural
  Information Processing Systems}, Red Hook, NY, USA, 2019. Curran Associates
  Inc.

\bibitem[Zhou et~al.(2021)Zhou, Du, and Wu]{zhou20213pvd}
Linqi Zhou, Yilun Du, and Jiajun Wu.
\newblock {3D} shape generation and completion through point-voxel diffusion.
\newblock In \emph{2021 IEEE/CVF International Conference on Computer Vision
  (ICCV)}, pages 5806--5815. IEEE, 2021.

\end{thebibliography}

\clearpage
\appendix
\section{Background}
\label{app:background_app}

\begin{proof}[Proof of Lemma~\ref{lem:algebraic}]
Both identities follow from linearity of the integral, in the integrand and in the measure respectively. For the first, $af+bg\in\Cont$ and
$$
P_{af+bg}(\mu)=\int_\Om (af+bg)\dd\mu=a\int_\Om f\dd\mu+b\int_\Om g\dd\mu=a\,P_f(\mu)+b\,P_g(\mu).
$$
For the second, $\mu_t:=(1-t)\mu_0+t\mu_1\in\PO$, and for every Borel $B\subseteq\Om$ we have $\mu_t(B)=(1-t)\mu_0(B)+t\mu_1(B)$; the identity $\int_\Om f\dd\mu_t=(1-t)\int_\Om f\dd\mu_0+t\int_\Om f\dd\mu_1$ therefore holds for simple $f$ and extends to all $f\in\Cont$ by uniform approximation and boundedness of $f$ on the compact $\Om$. Hence $P_f(\mu_t)=(1-t)P_f(\mu_0)+tP_f(\mu_1)$.
\end{proof}

\noindent\textbf{Details on the Slicing Kernel Assumption}
\label{app:kernel_assumption}

We provide additional details for Assumption~\ref{ass:kernel}. Let
$f\sim\sigma=\mathrm{GP}(0,\kappa)$, where $\kappa$ is continuous and
universal on the compact space $\Om$, i.e., its RKHS $\Hk$ is dense in
$\Cont$. A standard example is the Gaussian radial basis function (RBF)
kernel on compact $\Om\subset\R^d$
\citep{steinwart2001influence,micchelli2006universal}.

The canonical metric induced by the Gaussian process is
\[
d_\kappa(x,y)
:=
\left(
\mathbb{E}_{f\sim\sigma}
\big[(f(x)-f(y))^2\big]
\right)^{1/2}
=
\big(
\kappa(x,x)-2\kappa(x,y)+\kappa(y,y)
\big)^{1/2}.
\]
For a metric $d$, the covering number is defined as
\[
N(\Om,d,\varepsilon)
:=
\min\left\{
n\in\mathbb{N}:
\exists x_1,\ldots,x_n\in\Om
\text{ such that }
\Om=\bigcup_{i=1}^n B_d(x_i,\varepsilon)
\right\}.
\]
The sample-path continuity required by Assumption~\ref{ass:kernel}
is guaranteed by Dudley's entropy condition
\[
\int_0^1
\sqrt{\log N(\Om,d_\kappa,\varepsilon)}
\,\dd\varepsilon
<\infty,
\]
where $N$ denotes the covering number. 
Under this condition, the Gaussian process admits a version with almost
surely continuous sample paths
\citep{dudley1967sizes,adler2007random}. Universality remains a separate
requirement on $\kappa$.

\begin{lemma}%
\label{lem:gaussian-facts}
Under Assumption~\ref{ass:kernel}, (i) there is $\varepsilon>0$ with $\E_\sigma\exp(\varepsilon\norm{f}_\infty^2)<\infty$; in particular $\Lambda_r:=\E_\sigma\norm{f}_\infty^r<\infty$ for all $r\ge1$ \citep{fernique1970integrabilite,ledoux1991probability}; (ii) $\supp(\sigma)=\overline{\Hk}^{\norm{\cdot}_\infty}=\Cont$ \citep[Thm.~3.6.1]{bogachev1998gaussian}, see also \citep{vaart2008reproducing}.
\end{lemma}

\begin{remark}
\label{rem:facts}
More specifically, from (ii), we have that for any $f\in\Cont,\epsilon>0$, there exists finitely many $x_1,\ldots x_n\in\Omega, c_1,\ldots c_n\in\mathbb{R}$, with 
$\|f-\sum_{i=1}^nc_i\kappa(\cdot,x_i)\|_\infty<\epsilon.$
\end{remark}

\begin{assumption}%
\label{ass:smooth}
We assume that the underlying space $\Om\subset\R^d$ is compact and that $\kappa\in C^{2}(\Om\times\Om)$. Then the derivative process $\nabla f=(\partial_1f,\dots,\partial_df)$ is a centered Gaussian field with covariance
$$
\E_\sigma\big[\partial_i f(x)\,\partial_j f(y)\big]=\frac{\partial^2\kappa(x,y)}{\partial x_i\,\partial y_j},\qquad i,j=1,\dots,d,
$$
and we assume that its canonical metric
$$
d_\nabla(x,y):=\big(\E_\sigma\norm{\nabla f(x)-\nabla f(y)}^2\big)^{1/2}
$$
satisfies Dudley's entropy condition $\int_0^1\sqrt{\log N(\Om,d_\nabla,\varepsilon)}\,\dd\varepsilon<\infty$. Under these hypotheses $f\in C^1(\Om)$ for $\sigma$-a.e.\ $f$, i.e.\ $\sigma\big(C^1(\Om)\big)=1$, and the Borell--TIS inequality \citep{adler2007random} gives
$$
\E_\sigma\Big[\sup_{x\in\Om}\norm{\nabla f(x)}^{r}\Big]<\infty\qquad\text{for all }r\ge1 .
$$
For instance, the Gaussian RBF kernel is $C^\infty$ and satisfies these conditions, with $C^\infty$ sample paths.
\end{assumption}

\paragraph{Schoenberg theory.} A metric $d$ on a set $X$ is \emph{Hilbertian} if it is realized by an isometric feature map: there exist a Hilbert space $\mathcal{H}$ and a map $\phi:X\to\mathcal{H}$ with
$$
d(x,y)=\norm{\phi(x)-\phi(y)}_{\mathcal H}\qquad\text{for all }x,y\in X .
$$
A symmetric function $\psi:X\times X\to\R$ with $\psi(x,x)=0$ is \emph{conditionally negative definite} (CND) if
$$
\sum_{i,j=1}^{n}c_ic_j\,\psi(x_i,x_j)\;\le\;0
\qquad\text{whenever } n\in\N,\ x_1,\dots,x_n\in X,\ c\in\R^n,\ \sum_{i=1}^{n}c_i=0,
$$
that is, negative semi-definiteness restricted to the hyperplane of zero-sum coefficient vectors. Schoenberg's theorem states that $d$ is Hilbertian if and only if $d^2$ is CND, and that a nonnegative CND function stays CND under fractional powers: $\psi$ CND implies $\psi^{\alpha}$ CND for every $\alpha\in(0,1]$. Applying this with $\psi=d^2$ and $\alpha=\beta/2$ gives that $d^{\beta}$ is CND for every $\beta\in(0,2]$, and by the CND--PD correspondence
$$
d^\beta \text{ is CND}\iff \exp(-\gamma d^\beta)\text{ is positive definite for every }\gamma>0,
$$
we conclude that
$$
K_{\gamma,\beta}(x,y)=\exp\big(-\gamma\,d(x,y)^{\beta}\big)
$$
is a positive definite kernel on $X$ for all $\gamma>0$ and $\beta\in(0,2]$ \citep{schoenberg1938metric,berg1984harmonic}.

\section{Our method:
Position between kernel and transport lifts}
\label{sec:comparisons}

Writing $m_\mu:=\int\kappa(\cdot,x)\dd\mu(x)\in\Hk$, the Gaussian field
$\mu\mapsto P_f(\mu)$ has covariance
$\E_\sigma[P_f(\mu)P_f(\nu)]=\ip{m_\mu}{m_\nu}_{\Hk}$.
Thus VSW uses Gaussian linear probes of the embedded meta-measures, in the
sense of an isonormal Gaussian process on $\Hk$.
This connects the construction to Hilbert-space slicing
\citep{han2023sliced} and kernel expected sliced Wasserstein
\citep{naslidnyk2025kernel}. It does not identify their slicing laws:
our GP is $C(\Om)$-valued and need not be $\Hk$-valued, whereas a
Hilbert-valued Gaussian requires trace-class covariance.
Intrinsic sliced Wasserstein \citep{rustamov2023intrinsic} instead pushes
each inner measure through Laplacian eigenfunctions to one-dimensional laws;
VSW takes scalar expectations and then compares their outer distributions.

Define $\MMD_\kappa(\mu,\nu)=\norm{\int\kappa(\cdot,x)\dd(\mu-\nu)(x)}_{\Hk}$, which satisfies
\begin{equation}
\label{eq:mmd-var}
\MMD_\kappa^2(\mu,\nu)=\iint\kappa\,\dd(\mu-\nu)\otimes\dd(\mu-\nu)=\E_{f\sim\sigma}\big[\ip{f}{\mu-\nu}^2\big].
\end{equation}

Write $c_p:=\E|Z|^p=2^{p/2}\Gamma\big(\tfrac{p+1}{2}\big)/\sqrt{\pi}$ for $Z\sim\mathcal N(0,1)$ (so $c_2=1$, $c_1=\sqrt{2/\pi}$).

\begin{proposition}[Exact restriction to Diracs: $\VSW$ extends MMD]
\label{prop:dirac}
For all $\mu,\nu\in\PO$ and $p\in[1,\infty)$,
\[
\VSW_p\big(\delta_\mu,\delta_\nu\big)\;=\;c_p^{1/p}\,\MMD_\kappa(\mu,\nu).
\]
In particular the embedding $\mu\mapsto\delta_\mu$ is, up to the constant $c_p^{1/p}$, an isometry from $(\PO,\MMD_\kappa)$ into $(\PPO,\VSW_p)$.
\end{proposition}

\begin{proof}[Proof of Proposition~\ref{prop:dirac}]
Under $\delta_\mu$ the projected law is the point mass at $\ip{f}{\mu}$, so the slice value is $|\ip{f}{\mu-\nu}|$. The scalar $\ip{f}{\mu-\nu}$ is centered Gaussian with variance $\MMD_\kappa^2(\mu,\nu)$ by \eqref{eq:mmd-var}; taking $p$-th absolute moments and $p$-th roots gives the claim.
\end{proof}

\begin{proposition}[Upper bound: Wasserstein over MMD, tight on Diracs]
\label{prop:upper}
For all $\bbP,\bbQ\in\PPO$ and $p\in[1,\infty)$,
\[
\VSW_p(\bbP,\bbQ)\;\le\;c_p^{1/p}\;\WoW_{p,\MMD_\kappa}(\bbP,\bbQ),
\]
with equality when $\bbP,\bbQ$ are Dirac meta-measures. If in addition Assumption~\ref{ass:smooth} holds and $L_p:=\E_\sigma[\Lip(f)^p]^{1/p}<\infty$, then also $\VSW_p(\bbP,\bbQ)\le L_p\,\WoW_{p,\W_1}(\bbP,\bbQ)\le L_p\,\WoW_{p,\W_p}(\bbP,\bbQ)$.
\end{proposition}

\begin{proof}[Proof of Proposition~\ref{prop:upper}]
Let $\Gamma\in\Pi(\bbP,\bbQ)$ be arbitrary. For each $f$, $(P_f\times P_f)_\pf\Gamma$ couples the projected laws, so
$\W_p^p\big((P_f)_\pf\bbP,(P_f)_\pf\bbQ\big)\le\int|\ip{f}{\mu-\nu}|^p\dd\Gamma(\mu,\nu)$.
Integrate in $f$ and use Fubini:
\[
\VSW_p^p(\bbP,\bbQ)\le\int\Big(\E_\sigma|\ip{f}{\mu-\nu}|^p\Big)\dd\Gamma(\mu,\nu)
=c_p\int \MMD_\kappa(\mu,\nu)^p\dd\Gamma(\mu,\nu).
\]
Taking the infimum over $\Gamma$ proves the bound.
\end{proof}

\begin{proposition}[Lower bound: intensities in MMD]
\label{prop:lower}
For all $\bbP,\bbQ\in\PPO$ and all $p\in[1,\infty)$,
\[
\VSW_p(\bbP,\bbQ)\;\ge\;c_p^{1/p}\,\MMD_\kappa\big(m_\bbP,\,m_\bbQ\big),
\]
where $m_\bbP,m_\bbQ$ are the intensity measures, in particular, $m_\bbP(B)=\mathbb{E}_{\mu\sim\bbP}[\mu(B)],\forall B\in\mathcal{B}(\Om)$.
\end{proposition}

\begin{proof}[Proof of Proposition~\ref{prop:lower}]
For laws $\alpha,\beta$ on $\R$, any coupling $(X,Y)$ satisfies, by Jensen, $\E|X-Y|^p\ge|\E X-\E Y|^p$; hence $\W_p(\alpha,\beta)\ge|\operatorname{mean}(\alpha)-\operatorname{mean}(\beta)|$. The mean of $(P_f)_\pf\bbP$ is $\E_{\mu\sim\bbP}\ip{f}{\mu}=\ip{f}{m_\bbP}$ by Fubini. Therefore
$\VSW_p^p\ge\E_\sigma|\ip{f}{m_\bbP-m_\bbQ}|^p$. The scalar $\ip{f}{m_\bbP-m_\bbQ}$ is centered Gaussian with variance $\MMD_\kappa^2(m_\bbP,m_\bbQ)$ by \eqref{eq:mmd-var}.
\end{proof}

\begin{example}[Strictness: identical intensities, positive $\VSW$]
\label{ex:mixture}
Let $a\ne b\in\Om$ and consider
\[
\bbP=\delta_{\frac12(\delta_a+\delta_b)}\qquad\text{versus}\qquad
\bbQ=\tfrac12\,\delta_{\delta_a}+\tfrac12\,\delta_{\delta_b}.
\]
Both have intensity $m_\bbP=m_\bbQ=\frac12(\delta_a+\delta_b)$, so every comparison based on intensities---in particular the level-2 MMD of mean embeddings---vanishes, and Proposition~\ref{prop:lower} is uninformative. Fix $f\in\Cont$. Under $\bbP$ the projected law is the point mass at $s:=\frac12(f(a)+f(b))$; under $\bbQ$ it is $\frac12\delta_{f(a)}+\frac12\delta_{f(b)}$. The only coupling of a Dirac with another law is the product, and $f(a)-s=s-f(b)=\frac12(f(a)-f(b))$, so
\[
\W_p^p\big((P_f)_\pf\bbP,(P_f)_\pf\bbQ\big)=\tfrac12|f(a)-s|^p+\tfrac12|f(b)-s|^p=2^{-p}\,|f(a)-f(b)|^p .
\]
Since $f(a)-f(b)\sim\mathcal N\big(0,d_\kappa(a,b)^2\big)$, integrating in $f$ gives
\[
\VSW_p^p(\bbP,\bbQ)=2^{-p}\,\E_\sigma|f(a)-f(b)|^p=2^{-p}c_p\,d_\kappa(a,b)^p>0 ,
\]
the strict positivity because $d_\kappa(a,b)>0$ whenever $\kappa$ is strictly positive definite and $a\ne b$. Thus $\VSW$ sees the outer (transport) structure that intensity comparisons cannot; every slice with $f(a)\ne f(b)$ separates the two meta-measures.
\end{example}

\begin{remark}[Interpretation]
Propositions~\ref{prop:dirac}--\ref{prop:lower} identify $\VSW_p$ as \emph{the sliced Wassersteinization of MMD at the meta level}: its inner geometry is exactly $\MMD_\kappa$ (tight on Diracs), while its outer geometry is transport---strictly finer than the mean-embedding (intensity) comparison, as Example~\ref{ex:mixture} shows, yet never exceeding the full outer OT cost over the $\MMD_\kappa$ ground metric.
\end{remark}

\begin{corollary}[Universal characteristic kernels on meta-measures]
\label{cor:kernel}
For every $\gamma>0$ and $\beta\in(0,2]$, the kernel
\[
K_{\gamma,\beta}(\bbP,\bbQ)
=
\exp\big(-\gamma\,\VSW_2(\bbP,\bbQ)^\beta\big)
\]
is positive definite on $\PPO$. Moreover, it is continuous and
\emph{universal} on the compact metric space $(\PPO,\mathrm{weak})$,
hence characteristic. Consequently, the MMD induced by
$K_{\gamma,\beta}$ is a metric on the space of Borel probability
measures over $\PPO$.
\end{corollary}

\begin{proof}[Proof of Corollary~\ref{cor:kernel}]
\emph{Positive definiteness.} By Theorem~\ref{thm:hilbertian}, there is an isometric embedding $T:(\PPO,\VSW_2)\to H$ into a Hilbert space $H$ with
$\VSW_2(\bbP,\bbQ)=\|T(\bbP)-T(\bbQ)\|_H$.
Hence $\VSW_2^2$ is conditionally negative definite (CND).
For every $\beta\in(0,2]$, we have $\beta/2\in(0,1]$, so
$r\mapsto r^{\beta/2}$ is a Bernstein function on $[0,\infty)$.
Applying it to the nonnegative CND function $\VSW_2^2$ shows that
$(\VSW_2^2)^{\beta/2}=\VSW_2^{\beta}$ is CND.
Schoenberg's theorem then yields that
$K_{\gamma,\beta}(\bbP,\bbQ)=\exp\big(-\gamma\,\VSW_2(\bbP,\bbQ)^{\beta}\big)$
is positive definite for all $\gamma>0$ and $\beta\in(0,2]$.

\emph{Continuity.} By Proposition~\ref{thm:topology}, $\VSW_2$ induces the weak topology on $\PPO$, hence $(\bbP,\bbQ)\mapsto\VSW_2(\bbP,\bbQ)$ is continuous for the weak topology. Composing with $t\mapsto\exp(-\gamma t^{\beta})$ gives continuity of $K_{\gamma,\beta}$.

\emph{Universality and characteristicness.} The space $(\PPO,\text{weak})$ is compact. For $\beta=2$, $K_{\gamma,2}(\bbP,\bbQ)=\exp(-\gamma\|T(\bbP)-T(\bbQ)\|_H^2)$ is the restriction of a Gaussian kernel on the Hilbert space $H$ to the compact set $T(\PPO)$. Such a Gaussian kernel is universal on compact subsets of a separable Hilbert space (e.g.\ \citealp{christmann2010universal}); therefore $K_{\gamma,2}$ is universal on $\PPO$.

For $\beta\in(0,2)$, the function $\psi(t):=\exp(-\gamma t^{\beta/2})$ is completely monotone on $[0,\infty)$, so by Bernstein's theorem there exists a probability measure $\nu_\beta$ on $[0,\infty)$ such that $\psi(t)=\int_0^\infty e^{-st}\dd\nu_\beta(s)$ and $\nu_\beta((0,\infty))>0$. Substituting $t=\VSW_2^2(\bbP,\bbQ)$ gives
\[
K_{\gamma,\beta}(\bbP,\bbQ)=\int_0^\infty \exp\big(-s\,\VSW_2^2(\bbP,\bbQ)\big)\dd\nu_\beta(s).
\]
If a finite signed Borel measure $\eta$ on $\PPO$ satisfies
$\iint K_{\gamma,\beta}\dd\eta\otimes\dd\eta=0$, Fubini applies because
the kernels are bounded and $\eta$ has finite total variation.
The nonnegative Gaussian-kernel quadratic forms therefore vanish for
$\nu_\beta$-a.e.\ $s$, hence for some $s>0$.
Universality of $\exp(-s\,\VSW_2^2)$ forces $\eta=0$ by the equivalence
between universality and integral strict positive definiteness on compact
spaces \citep[Prop.~4]{sriperumbudur2011universality}.
The same equivalence proves universality of $K_{\gamma,\beta}$;
universal continuous kernels on compact metric spaces are characteristic.
\end{proof}

\begin{remark}[Uses]
Corollary~\ref{cor:kernel} yields kernel two-sample tests between
collections of datasets, distribution regression whose inputs are
meta-measures (one level above \citealp{szabo2016learning}), and
Gaussian-process priors over functionals of meta-measures.
\end{remark}

\section{Proofs of the Statements in Section~\ref{sec:vsw}}
\label{sec:proofs:vsw}

\begin{proof}[Proof of Lemma~\ref{lem:regularity}]
(i) Let $(f_n,\mu_n)\to(f,\mu)$ in $\Cont\times\PO$. In particular:
\begin{itemize}
\item $f_n\to f$ in $(\Cont,\norm{\cdot}_\infty)$, i.e.\ $\norm{f_n-f}_\infty=\sup_{x\in\Om}|f_n(x)-f(x)|\to0$, which is uniform convergence on $\Om$;
\item $\mu_n\to\mu$ in $\PO$, i.e.\ $\mu_n\rightharpoonup\mu$ weakly: $\int_\Om g\dd\mu_n\to\int_\Om g\dd\mu$ for every $g\in\Cont$.
\end{itemize}
Then
\[
\begin{aligned}
\left|\ip{f_n}{\mu_n}-\ip{f}{\mu}\right|
&\le
\left|\ip{f_n-f}{\mu_n}\right|
+
\left|\ip{f}{\mu_n}-\ip{f}{\mu}\right| \le
\norm{f_n-f}_\infty
+
\left|\ip{f}{\mu_n}-\ip{f}{\mu}\right|.
\end{aligned}
\]
The first term tends to zero by uniform convergence, while the second tends to zero by weak convergence, since the fixed function $f\in\Cont$ is bounded and continuous. Hence $P$ is jointly continuous.

\noindent (ii) Choose fixed $f,g\in\Cont$, the deterministic map $\mu\mapsto(P_f(\mu),P_g(\mu))$ pushes $\bbP$ to a coupling between $(P_f)_\pf\bbP$ and $(P_g)_\pf\bbP$. We have:
\begin{align}
\E_{\mu\sim\bbP}|P_f(\mu)-P_g(\mu)|^p&=\mathbb{E}_{\mu\sim\mathbb{P}}|\int (f-g)d\mu|^p\nonumber\\
&\leq \mathbb{E}_{\mu\sim\mathbb{P}}\int|f-g|^p d\mu \nonumber\\
&\leq \sup_{x\in\Om}|f(x)-g(x)|^p\nonumber\\
&=\norm{f-g}_\infty^p;\nonumber
\end{align}
hence $\W_p\big((P_f)_\pf\bbP,(P_g)_\pf\bbP\big)\le\norm{f-g}_\infty$, and similarly for $\bbQ$. The triangle inequality for $\W_p$ on $\mathcal{P}_p(\R)$ then gives
$\big|\W_p((P_f)_\pf\bbP,(P_f)_\pf\bbQ)-\W_p((P_g)_\pf\bbP,(P_g)_\pf\bbQ)\big|\le2\norm{f-g}_\infty$. Thus $f\mapsto \W_p\big((P_f)_\pf\bbP,(P_f)_\pf\bbQ\big)$ is $2$-Lipschitz.

\emph{Finiteness.} Fix $f\in\Cont$. As noted after Definition~\ref{def:vsw}, both $(P_f)_\pf\bbP$ and $(P_f)_\pf\bbQ$ are supported in the interval $I_f:=[-\norm{f}_\infty,\norm{f}_\infty]$, of length $|I_f|=2\norm{f}_\infty$. Consequently every coupling $\pi\in\Pi\big((P_f)_\pf\bbP,(P_f)_\pf\bbQ\big)$ is supported in $I_f\times I_f$, on which $|x-y|\le2\norm{f}_\infty$, so
$$
\W_p^p\big((P_f)_\pf\bbP,(P_f)_\pf\bbQ\big)=\inf_{\pi}\int_{\R\times\R}|x-y|^p\dd\pi(x,y)\;\le\;\big(2\norm{f}_\infty\big)^p .
$$
Integrating this bound against $\sigma$ and using Lemma~\ref{lem:gaussian-facts}(i),
$$
\VSW_p^p(\bbP,\bbQ)=\int_{\Cont}\W_p^p\big((P_f)_\pf\bbP,(P_f)_\pf\bbQ\big)\dd\sigma(f)\;\le\;2^p\int_{\Cont}\norm{f}_\infty^p\dd\sigma(f)\;=\;2^p\Lambda_p\;<\;\infty .
$$
The integral in \eqref{eq:vsw} is therefore well defined.
\end{proof}

\begin{proof}[Proof of Theorem~\ref{thm:metric}]
We first show that  $\VSW_p$ is a pseudometric on $\PPO$: it is finite, symmetric, vanishes on the diagonal, and satisfies the triangle inequality.

\emph{Finiteness} was established in Lemma~\ref{lem:regularity}(ii): $\VSW_p^p(\bbP,\bbQ)\le2^p\Lambda_p<\infty$ for all $\bbP,\bbQ\in\PPO$.

\emph{Symmetry} is inherited slicewise: $\W_p$ is symmetric on $\mathcal{P}(\R)$.

\emph{Vanishing on the diagonal}: if $\bbP=\bbQ$ then $(P_f)_\pf\bbP=(P_f)_\pf\bbQ$ for every $f\in\Cont$.

\emph{Triangle inequality.} Let $\bbP,\bbQ,\mathbb{R}'\in\PPO$ and write $h_{\bbP\bbQ}(f):=\W_p\big((P_f)_\pf\bbP,(P_f)_\pf\bbQ\big)$, so that $\VSW_p(\bbP,\bbQ)=\norm{h_{\bbP\bbQ}}_{L^p(\sigma)}$ by \eqref{eq:vsw}. For each fixed $f$, the triangle inequality for $\W_p$ on $\mathcal{P}(\R)$ gives the pointwise bound
$$
h_{\bbP\bbQ}(f)\;\le\;h_{\bbP\mathbb{R}'}(f)+h_{\mathbb{R}'\bbQ}(f).
$$
Taking $L^p(\sigma)$ norms and applying Minkowski's inequality gives the result.

 The only property left to verify is identity of indiscernibles. Assume $\VSW_p(\bbP,\bbQ)=0$.

Define
\[
A:=\big\{f\in\Cont:\ (P_f)_\pf\bbP=(P_f)_\pf\bbQ\big\}
=\big\{f\in\Cont:\ \W_p\big((P_f)_\pf\bbP,(P_f)_\pf\bbQ\big)=0\big\}.
\]
The two descriptions agree because both projected laws are compactly supported, hence lie in $\mathcal{P}_p(\R)$, the subspace of $\mathcal{P}(\R)$ of measures with finite $p$-moment, on which $\W_p$ is a metric.

\emph{Step 0 ($\sigma(A)=1$).} By assumption and \eqref{eq:vsw},
\[
0=\VSW_p^p(\bbP,\bbQ)=\int_{\Cont}\W_p^p\big((P_f)_\pf\bbP,(P_f)_\pf\bbQ\big)\dd\sigma(f).
\]
The integrand is nonnegative and Borel (Lemma~\ref{lem:regularity}(ii)), so it vanishes $\sigma$-a.e., i.e.\ $\sigma(A)=1$.

\emph{Step 1 ($A=\Cont$).} The map $h(f):=\W_p\big((P_f)_\pf\bbP,(P_f)_\pf\bbQ\big)$ is continuous by Lemma~\ref{lem:regularity}(ii), hence $A=h^{-1}(\{0\})$ is closed. Since $\sigma(A)=1$, closedness implies $\supp\sigma\subseteq A$. By Lemma~\ref{lem:gaussian-facts}(ii) and universality (Assumption~\ref{ass:kernel}), $\supp\sigma=\Cont$, hence $A=\Cont$. Therefore
\begin{equation}
\label{eq:all-f-app}
(P_f)_\pf\bbP=(P_f)_\pf\bbQ\qquad\text{for every }f\in\Cont.
\end{equation}

\emph{Step 2 (finite-dimensional distributions agree).} Fix $k\in\N$ and $f_1,\dots,f_k\in\Cont$. Consider $X:\PO\to\R^k$, $X(\mu)=(P_{f_1}(\mu),\dots,P_{f_k}(\mu))$. For any $t\in\R^k$, Lemma~\ref{lem:algebraic} gives
\[
\ip{t}{X(\mu)}=\sum_{j=1}^k t_j P_{f_j}(\mu)=P_{g_t}(\mu),\qquad g_t:=\sum_{j=1}^k t_j f_j\in\Cont.
\]
Applying \eqref{eq:all-f-app} to $g_t$ shows that $\ip{t}{X}$ has the same law under $\mu\sim\bbP$ and under $\mu\sim\bbQ$. By the Cram\'er--Wold theorem, the laws of $X$ under $\bbP$ and $\bbQ$ coincide, i.e.
\begin{equation}
\label{eq:fdd-app}
\big(P_{f_1},\dots,P_{f_k}\big)_\pf\bbP=\big(P_{f_1},\dots,P_{f_k}\big)_\pf\bbQ\qquad\text{in }\mathcal{P}(\R^k).
\end{equation}

\emph{Step 3 ($\bbP=\bbQ$).} Let
\[
\mathcal{C}=\Big\{\{\mu\in\PO:(P_{f_1}(\mu),\dots,P_{f_k}(\mu))\in B\}:\ k\in\N,\ f_1,\dots,f_k\in\Cont,\ B\in\mathcal{B}(\R^k)\Big\}.
\]
Then $\mathcal{C}$ is a $\pi$-system (closed under finite intersections). Moreover $\sigma(\mathcal{C})=\mathcal{B}(\PO)$ because the weak Borel $\sigma$-algebra is generated by the maps $P_f$ (Section~\ref{sec:background}).

Define $\mathcal{L}:=\{C\in\mathcal{B}(\PO):\bbP(C)=\bbQ(C)\}$. This is a $\lambda$-system. By \eqref{eq:fdd-app}, we have $\mathcal{C}\subseteq\mathcal{L}$. Dynkin's $\pi$--$\lambda$ theorem yields $\sigma(\mathcal{C})\subseteq\mathcal{L}$, hence $\mathcal{B}(\PO)\subseteq\mathcal{L}$ and therefore $\bbP=\bbQ$.
\end{proof}

\begin{remark}%
\label{rem:universality}
Only Step~1 uses universality, and it does so through the support theorem: the closed agreement set $A$ must contain $\supp\sigma$. If $\kappa$ is not universal, Steps~0--3 still yield $(P_f)_\pf\bbP=(P_f)_\pf\bbQ$ for every $f\in\overline{\Hk}$, hence $\bbP=\bbQ$ on $\mathcal{F}_\kappa:=\sigma\big(P_f:f\in\overline{\Hk}\big)\subseteq\mathcal{B}(\PO)$. Thus $\VSW_p$ remains a pseudometric on $\PPO$ whose null relation is exactly agreement on $\mathcal{F}_\kappa$, and it is a metric on the corresponding quotient; when $\mathcal{F}_\kappa\subsetneq\mathcal{B}(\PO)$, distinct meta-measures can be at $\VSW_p$-distance zero. Conversely, any slicing measure with full support in $\Cont$ (not necessarily Gaussian) and finite $\Lambda_p$ yields a metric by the same proof; the Gaussian choice is what later gives exact MMD compatibility (Propositions~\ref{prop:dirac}--\ref{prop:lower}) and Borell--TIS moment control.
\end{remark}

\begin{proof}[Proof of Proposition~\ref{thm:topology}]
\emph{Weak convergence implies $\VSW_p\to0$.} Let $\bbP_m\rightharpoonup\bbP$ in the weak topology on $\PPO$, and fix $f\in\Cont$. Since $P_f:\PO\to\R$ is continuous (Lemma~\ref{lem:regularity}(i)), for every $g\in C_b(\R)$ the composition $g\circ P_f$ belongs to $C(\PO)$; hence
\[
\int_\R g\dd\big((P_f)_\pf\bbP_m\big)=\int_{\PO}g\circ P_f\dd\bbP_m\;\longrightarrow\;\int_{\PO}g\circ P_f\dd\bbP=\int_\R g\dd\big((P_f)_\pf\bbP\big),
\]
i.e.\ $(P_f)_\pf\bbP_m\rightharpoonup(P_f)_\pf\bbP$ in $\mathcal{P}(\R)$. Moreover $|P_f|\le\norm{f}_\infty$ on $\PO$, so all these laws are supported in the fixed compact interval $I_f=[-\norm{f}_\infty,\norm{f}_\infty]$. On probability measures supported in a common compact interval, $\W_p$ metrizes weak convergence, so
\[
\W_p^p\big((P_f)_\pf\bbP_m,(P_f)_\pf\bbP\big)\;\longrightarrow\;0\qquad\text{for each fixed }f\in\Cont.
\]
The integrand is dominated by $(2\norm{f}_\infty)^p$, which lies in $L^1(\sigma)$ by Lemma~\ref{lem:gaussian-facts}(i). Dominated convergence therefore yields $\VSW_p(\bbP_m,\bbP)\to0$.

\emph{$\VSW_p\to0$ implies weak convergence.} Consider the identity map
\[
\iota:(\PPO,\text{weak})\to(\PPO,\VSW_p),
\]
which is a bijection. The previous paragraph shows it is continuous. Since $(\PPO,\text{weak})$ is compact and $(\PPO,\VSW_p)$ is Hausdorff (being metric), a continuous bijection from a compact space onto a Hausdorff space is a homeomorphism. Hence $\VSW_p$ induces the weak topology.
\end{proof}

\begin{proof}[Proof of Theorem~\ref{thm:hilbertian}]
\emph{Measurability.} The evaluation map $(f,\mu)\mapsto P_f(\mu)$ is jointly continuous (Lemma~\ref{lem:regularity}(i)). Hence, for fixed $\bbP\in\PPO$, the function
\[
G(f,s):=\bbP\{\mu\in\PO:P_f(\mu)\le s\}
\]
is jointly Borel in $(f,s)$: the set $\{(f,s,\mu):P_f(\mu)-s\le0\}$ is closed in $\Cont\times\R\times\PO$, so its indicator is jointly Borel, and Tonelli (integrating out $\mu$ against $\bbP$) yields joint measurability of $G$. Define the quantile
\[
q(f,u):=\inf\{s\in\bbQ_{\mathrm{rat}}:G(f,s)\ge u\}.
\]
Since the infimum runs over a countable set and each subgraph $\{(f,u):u\le G(f,s)\}$ is Borel, $q$ is jointly Borel in $(f,u)$.

\emph{Integrability.} Since $(P_f)_\pf\bbP$ is supported in $[-\norm{f}_\infty,\norm{f}_\infty]$, we have $|q(f,u)|\le\norm{f}_\infty$ for all $(f,u)$. By Lemma~\ref{lem:gaussian-facts}(i), $\E_\sigma\norm{f}_\infty^2<\infty$, so $T(\bbP)\in L^2(\sigma\otimes\lambda)$.

\emph{Isometry.} For $p=2$, the one-dimensional OT formula \eqref{eq:1dot} gives, for each $f$,
\[
\W_2^2\big((P_f)_\pf\bbP,(P_f)_\pf\bbQ\big)=\int_0^1\big|q_{(P_f)_\pf\bbP}(u)-q_{(P_f)_\pf\bbQ}(u)\big|^2\dd u.
\]
Integrating in $f$ and applying Tonelli yields
\[
\VSW_2^2(\bbP,\bbQ)=\int_{\Cont}\!\int_0^1\big|T(\bbP)(f,u)-T(\bbQ)(f,u)\big|^2\dd u\dd\sigma(f)=\norm{T(\bbP)-T(\bbQ)}_{L^2(\sigma\otimes\lambda)}^2.
\]
This proves the claimed isometric embedding.
\end{proof}

\begin{proof}[Proof of Theorem \ref{thm:sample}]
By the triangle inequality (Theorem \ref{thm:metric}) and $(a+b)^p\le2^{p-1}(a^p+b^p)$,
\[
\E\,\VSW_p^p(\bbP,\widetilde\bbP)\le2^{p-1}\Big(\E\,\VSW_p^p(\bbP,\widehat\bbP_N)+\E\,\VSW_p^p(\widehat\bbP_N,\widetilde\bbP)\Big).
\]
\emph{Outer term.} Fix $f$ and let $\nu_f:=(P_f)_\pf\bbP\in\mathcal P(I_f)$, $I_f=[-\norm{f}_\infty,\norm{f}_\infty]$. Then $(P_f)_\pf\widehat\bbP_N$ is exactly the empirical measure of the $N$ i.i.d.\ scalars $P_f(\mu_i)\sim\nu_f$. Using $\W_p^p\le|I_f|^{p-1}\W_1$ on $I_f$ and \eqref{eq:1drate},
\[
\E\,\W_p^p\big((P_f)_\pf\widehat\bbP_N,\nu_f\big)\le(2\norm{f}_\infty)^{p-1}\cdot\frac{2\norm{f}_\infty}{2\sqrt N}=\frac{(2\norm{f}_\infty)^p}{2\sqrt N}.
\]
Integrating in $f$ (Tonelli; joint measurability as before) gives $$\E\,\VSW_p^p(\bbP,\widehat\bbP_N)\le \tfrac12 \cdot 2^p\Lambda_p N^{-1/2}.$$

\emph{Inner term.} Pair $\mu_i$ with $\widehat\mu_i$: the coupling of $(P_f)_\pf\widehat\bbP_N$ and $(P_f)_\pf\widetilde\bbP$ that matches index $i$ to index $i$ yields
\[
\W_p^p\big((P_f)_\pf\widehat\bbP_N,(P_f)_\pf\widetilde\bbP\big)\le\frac1N\sum_{i=1}^N\big|P_f(\mu_i)-P_f(\widehat\mu_i)\big|^p .
\]
Conditionally on $\mu_i$, $P_f(\widehat\mu_i)-P_f(\mu_i)=\frac1n\sum_j\big(f(x_{ij})-\E f(x_{ij})\big)$ is a centered average of i.i.d.\ variables bounded by $2\norm{f}_\infty$. For $p\in[1,2]$, Jensen and the variance bound give $\E\big|\cdot\big|^p\le\big(\tfrac{(2\norm{f}_\infty)^2}{4n}\big)^{p/2}\le(2\norm{f}_\infty)^p n^{-p/2}$; for $p>2$, the Marcinkiewicz--Zygmund inequality gives $\E|\cdot|^p\le C_p(2\norm{f}_\infty)^pn^{-p/2}$. Taking expectations, then the $\sigma$-integral (Tonelli), yields $$\E\,\VSW_p^p(\widehat\bbP_N,\widetilde\bbP)\le C_p\,2^p\Lambda_p\,n^{-p/2}.$$ 
Combining the outer and inner bounds proves the theorem. The distance rates
in Remark~\ref{rem:sample-rates} follow by Jensen
($\E\VSW_p\le(\E\VSW_p^p)^{1/p}$) and subadditivity of
$x\mapsto x^{1/p}$.
\end{proof}

\begin{remark}%
\label{rem:outer-rate}
For $p=2$, Theorem~\ref{thm:sample} gives
$\E\,\VSW_2(\bbP,\widetilde\bbP_{N,n})=O(N^{-1/4}+n^{-1/2})$,
whereas the hierarchical IPM of \citet{catalano2024hierarchical}
attains the parametric outer rate. The outer exponent is sharp in
the worst case, as a two-atom meta-measure shows. For distinct
$\mu_a,\mu_b\in\PO$, set
$\bbP=\tfrac12\delta_{\mu_a}+\tfrac12\delta_{\mu_b}$,
and let $\widehat w_N$ be the proportion of $\mu_a$ among the
$N$ outer samples. Matching the common mass in each projected law gives
\[
\VSW_2^2(\bbP,\widehat\bbP_N)
=\left|\widehat w_N-\tfrac12\right|
\MMD_\kappa^2(\mu_a,\mu_b).
\]
Since $N\widehat w_N\sim\mathrm{Binomial}(N,1/2)$, its centered
absolute moments satisfy
$\E|\widehat w_N-\tfrac12|\asymp N^{-1/2}$ and
$\E|\widehat w_N-\tfrac12|^{1/2}\asymp N^{-1/4}$.
Thus, for this fixed $\bbP$, the expected squared distance has
order $N^{-1/2}$ and the expected distance has order $N^{-1/4}$.
These are worst-case rates; faster outer convergence is possible
under additional regularity of the projected laws.
Specifically, if for $\sigma$-a.e.\ $f$ the law $(P_f)_\pf\bbP$
admits a density $\varrho_f$ with
\[
J_2(f):=\int_\R\frac{F_{(P_f)_\pf\bbP}(z)\big(1-F_{(P_f)_\pf\bbP}(z)\big)}{\varrho_f(z)}\dd z,
\qquad \E_\sigma\big[J_2(f)\big]<\infty,
\]
then $\E\,\W_2^2\big((P_f)_\pf\bbP,(P_f)_\pf\widehat\bbP_N\big)
\le 2J_2(f)/(N+1)$ \citep[Thm.~5.1]{bobkov2019one}.
Integrating in $f$ yields
$\E\,\VSW_2^2(\bbP,\widehat\bbP_N)
\le 2\E_\sigma[J_2]/(N+1)$, hence
$\E\,\VSW_2(\bbP,\widehat\bbP_N)=O(N^{-1/2})$.
Combining this with the unchanged inner-sampling bound gives
$\E\,\VSW_2(\bbP,\widetilde\bbP_{N,n})
=O(N^{-1/2}+n^{-1/2})$ under this additional regularity condition.
In the definition of $J_2$, we use the conventions $0/0=0$
and $c/0=+\infty$ for $c>0$.
\end{remark}

\begin{proof}[Proof of Proposition \ref{prop:mc}]
The summands are i.i.d.\ with mean $\VSW_p^p$ and second moment at most $\E_\sigma(2\norm f_\infty)^{2p}=2^{2p}\Lambda_{2p}$; apply $\E|X-\E X|\le\sqrt{\operatorname{Var}X}$ to the empirical mean. (Cf.\ the analogous bound for sliced divergences in \citealp{nadjahi2020statistical}.)
\end{proof}

\begin{remark}[Finite truncations are pseudometrics]
\label{rem:kl-pseudo}
For finite $R$ the support of $\sigma_R$ is the finite-dimensional subspace $\mathrm{span}\{\varphi_1,\dots,\varphi_R\}\subseteq\Cont$. When this subspace is proper, Step~1 of the proof of Theorem~\ref{thm:metric} does not apply; $\VSW_p^{(\sigma_R)}$ is in general only a pseudometric, with null relation given by agreement of the projected laws along that subspace (cf.\ Remark~\ref{rem:universality}). Lemma~\ref{lem:kl} therefore describes a sequence of pseudometrics converging uniformly to the metric $\VSW_p^{(\sigma)}$; separation is guaranteed in the limit $R\to\infty$, and may already hold at finite $R$ when the feature map separates inner measures.
\end{remark}

\begin{remark}[Random features are not covered by the coupling bound]
\label{rem:rff-coupling}
Proposition~\ref{prop:stability} requires a coupling of the two slicing measures. An independent random Fourier feature draw \eqref{eq:rff} is \emph{not} naturally coupled to the limiting process $f\sim\sigma$, so Lemma~\ref{lem:kl} does not apply to it; the covariance of \eqref{eq:rff} converges uniformly to $\kappa$ \citep{rahimi2007random}, but a quantitative transfer from covariance convergence to convergence of the sliced distances is left open, and we treat the induced bias empirically.
\end{remark}

\paragraph{The shared finite-feature implementation.}
\label{app:finite-features}
Condition on the sampled Fourier basis, let $\sigma_R$ be the conditional
law of $f_R$, and define
$\phi_R(x)=\sqrt{2/R}\big(\cos(\ip{\omega_r}{x}+b_r)\big)_{r=1}^R$
and $z_R(\mu)=\int\phi_R(x)\dd\mu(x)$.
Then $P_{f_R}(\mu)=a^\top z_R(\mu)$ for $a\sim\mathcal N(0,I_R)$.
Writing $\SW_p$ for classical slicing with uniform unit directions,
Gaussian radial symmetry gives the exact identity
\[
\VSW_p^{(\sigma_R)}(\bbP,\bbQ)^p
=\E\norm{a}^p\,
\SW_p^p\big((z_R)_\pf\bbP,(z_R)_\pf\bbQ\big).
\]
The distance itself has factor $(\E\norm{a}^p)^{1/p}$, equal to
$\sqrt R$ for $p=2$. The finite-feature map can send distinct measures to the same vector,
so this is generally a pseudometric on the original meta-measures.
Conditional on the basis, Proposition~\ref{prop:mc}'s variance argument
applies with $\sigma_R$ and its sup-norm moments in place of $\sigma$.
Increasing $L$ controls this conditional Monte Carlo error, while
no quantitative feature-error bound is provided here (Remark~\ref{rem:rff-coupling}).
For a direct finite-feature reference check, Proposition~\ref{prop:dirac}
gives $\VSW_2(\delta_\mu,\delta_\nu)=\MMD_\kappa(\mu,\nu)$,
whereas the conditional finite-feature value is
$\norm{z_R(\mu)-z_R(\nu)}_2$.

\paragraph{Gradient Flow}%

\begin{lemma}[No ties, almost surely]
\label{lem:ties}
Suppose $\kappa$ is characteristic and the clouds $\widehat\mu_1,\dots,\widehat\mu_N$ are pairwise distinct as measures (and similarly for the targets). Then for $\sigma$-a.e.\ $f$, the projected scalars $s_1(f),\dots,s_N(f)$ are pairwise distinct (and similarly $t_1,\dots,t_N$), so the sorting permutation $\pi_f$ matching $s_{(i)}$ to $t_{(i)}$ is locally constant in $X$ around any configuration with distinct clouds.
\end{lemma}

\begin{proof}
For $i\ne i'$, $s_i(f)-s_{i'}(f)=\ip{f}{\widehat\mu_i-\widehat\mu_{i'}}$ is a centered Gaussian with variance $\MMD_\kappa^2(\widehat\mu_i,\widehat\mu_{i'})$, which is strictly positive because $\kappa$ is characteristic and the clouds are distinct; a nondegenerate Gaussian hits $0$ with probability zero. A union over the finitely many pairs concludes.
\end{proof}

\begin{algorithm}[t]
\caption{$\VSW$ flow for populations of point clouds (one step)}
\label{alg:flow}
\begin{algorithmic}[1]
\Require particles $X=(x_{ij})$, targets $\{y_{kl}\}$, kernel $\kappa$, slices $L$, features $R$, step $\eta$
\State sample fresh slices $f_1,\dots,f_L$ via \eqref{eq:rff}
\For{$l=1,\dots,L$}
  \State compute $s_i^{(l)},t_k^{(l)}$; sort both; obtain matching $\pi_l$
\EndFor
\State $g_{ij}\gets\frac{2}{NnL}\sum_{l}\big(s_i^{(l)}-t^{(l)}_{\pi_l(i)}\big)\nabla f_l(x_{ij})$
\State $x_{ij}\gets x_{ij}-\eta\,g_{ij}$ \Comment{or pass $g$ to Adam / reparametrized generator}
\end{algorithmic}
\end{algorithm}

\begin{remark}[Parametric generators.]
For a differentiable generator
$x_{ij}=g_\theta(z_i,\zeta_{ij})$, the particle gradients can
be backpropagated through the generator by the chain rule,
wherever the empirical objective is differentiable.
For a fixed kernel and empirical minibatch, independent
Gaussian-process slices give unbiased Monte Carlo estimates
of the empirical loss and, when differentiation and expectation
can be interchanged, its gradient. With finite random Fourier
features, the corresponding statement holds for the
finite-feature objective conditional on the sampled basis.
Under suitable differentiation and integrability assumptions,
averaging over random minibatches targets the gradient of
the expected empirical minibatch objective, which generally
differs from the population $\VSW_2^2$ objective.
Theorem~\ref{thm:sample} controls distance estimation error;
it does not establish convergence of the corresponding gradients.
\end{remark}

Finally, we prove Proposition \ref{prop:gradient} which holds under Assumptions~\ref{ass:kernel}--\ref{ass:smooth} and the hypotheses of Lemma~\ref{lem:ties}.

\begin{proof}[Proof of Proposition~\ref{prop:gradient}]

For $f\in C^1$ write
\[
F_f(X):=\min_{\pi\in S_N}\frac1N\sum_{i=1}^N\big(s_i(f)-t_{\pi(i)}(f)\big)^2,
\]
so that $\mathcal F=\E_\sigma[F_f]$ by \eqref{eq:1dot} (the increasing rearrangement realizes the minimum). Fix a coordinate $x_{ij}$ and let $\Om'\supset\Om$ be a fixed compact neighborhood containing all configurations under consideration.
In this proof, $\norm{f}_\infty$ denotes $\sup_{x\in\Om'}|f(x)|$.
We assume that the regularity and moment conditions of
Assumption~\ref{ass:smooth} hold on $\Om'$; then the field $\nabla f$ is defined and Borell--TIS applies on $\Om'$ (for the Gaussian RBF, $f$ extends to all of $\R^d$). The interchange of one partial derivative with $\E_\sigma$ proceeds in three steps.

\emph{Step 1 (pointwise differentiability for $\sigma$-a.e.\ $f$).} Let $A$ be the full-measure set of Lemma~\ref{lem:ties}. For $f\in A$ the projected scalars $s_1(f),\dots,s_N(f)$ (resp.\ $t_1(f),\dots,t_N(f)$) are pairwise distinct, so the minimizing permutation $\pi_f$ is unique and remains optimal on a ball $B(X,r_f)$ whose radius $r_f>0$ \emph{depends on $f$}, by continuity of $X\mapsto s_i(f)$. On this ball $F_f$ coincides with the smooth function $\frac1N\sum_i(s_i-t_{\pi_f(i)})^2$, whence
\[
\partial_{x_{ij}}F_f(X)=\tfrac{2}{Nn}\big(s_i(f)-t_{\pi_f(i)}(f)\big)\nabla f(x_{ij}),
\qquad \partial s_i/\partial x_{ij}=\tfrac1n\nabla f(x_{ij}).
\]

\emph{Step 2 (uniform integrable envelope for the difference quotients).} For every $f\in C^1$ (not only $f\in A$) and every fixed $\pi\in S_N$, the map $X\mapsto\frac1N\sum_i(s_i-t_{\pi(i)})^2$ is differentiable with partial gradient in $x_{ij}$ bounded in norm by
\[
E(f):=\tfrac{4}{Nn}\,\norm{f}_\infty\,\sup_{\Om'}\norm{\nabla f},
\]
since $|s_i-t_{\pi(i)}|\le2\norm{f}_\infty$. A pointwise minimum of finitely many functions that are $M$-Lipschitz along a segment is itself $M$-Lipschitz along that segment; hence every difference quotient of $F_f$ in the coordinate $x_{ij}$ is bounded by $E(f)$, \emph{uniformly in the step size}. Moreover $E\in L^1(\sigma)$: by Cauchy--Schwarz,
\[
\E_\sigma\big[\norm{f}_\infty\,\sup_{\Om'}\norm{\nabla f}\big]
\le\big(\E_\sigma\norm{f}_\infty^2\big)^{1/2}\big(\E_\sigma\sup_{\Om'}\norm{\nabla f}^2\big)^{1/2}<\infty
\]
by Lemma~\ref{lem:gaussian-facts}(i) and Assumption~\ref{ass:smooth}.

\emph{Step 3 (dominated convergence).} As $h\to0$, the difference quotients $h^{-1}\big(F_f(X+he)-F_f(X)\big)$ converge for $\sigma$-a.e.\ $f$ (Step~1) and are dominated by $E\in L^1(\sigma)$ (Step~2); dominated convergence yields the partial derivative of $\mathcal F$ and formula \eqref{eq:gradient}.
Finally, the right-hand side of \eqref{eq:gradient} is continuous at
each configuration $X$ with pairwise distinct clouds. Indeed, if
$X_r\to X$, then for the full-measure set of slices having no ties at
$X$, continuity of the projected scalars implies
$\pi_f(X_r)=\pi_f(X)$ for all sufficiently large $r$. Hence the
integrand in \eqref{eq:gradient} converges pointwise, and the same
integrable envelope $E$ allows another application of dominated
convergence. Thus, all partial derivatives of $\mathcal F$ exist and are continuous near $X$, and $\mathcal F$ is differentiable at $X$.
\end{proof}

\section{Statistical Properties}\label{sec: stats}
\begin{corollary}[Error of the fully empirical estimator]
\label{cor:two-sample}
Let $\bbP,\bbQ\in\PPO$ be observed through the two-stage scheme with parameters $(N,n)$ and $(M,m)$ respectively, and let $\widehat{\VSW}_{p,L}$ denote the $L$-slice Monte Carlo estimator of Proposition~\ref{prop:mc}. Then
\[
\begin{aligned}
&\E\Big|\widehat{\VSW}_{p,L}\big(\widetilde\bbP_{N,n},\widetilde\bbQ_{M,m}\big)-\VSW_p(\bbP,\bbQ)\Big|
\\
&\qquad\le C\Big(N^{-1/(2p)}+n^{-1/2}+M^{-1/(2p)}+m^{-1/2}+L^{-1/(2p)}\Big),
\end{aligned}
\]
with $C=C\big(p,\Lambda_{2p}\big)$.
\end{corollary}

\begin{proof}
By the triangle inequality for $\VSW_p$ (Theorem~\ref{thm:metric}),
$\big|\VSW_p(\widetilde\bbP_{N,n},\widetilde\bbQ_{M,m})-\VSW_p(\bbP,\bbQ)\big|\le\VSW_p(\bbP,\widetilde\bbP_{N,n})+\VSW_p(\bbQ,\widetilde\bbQ_{M,m})$,
and Theorem~\ref{thm:sample} bounds the expectation of each term after taking
$p$-th roots and applying Jensen. For the slicing error, Proposition~\ref{prop:mc} gives $\E|\widehat{\VSW}_{p,L}^p-\VSW_p^p|\le 2^p(\Lambda_{2p}/L)^{1/2}$ at the empirical meta-measures, and $|a^{1/p}-b^{1/p}|\le|a-b|^{1/p}$ for $a,b\ge0$ converts this into a bound on the estimator itself. Combining the three contributions gives the claim.
\end{proof} 

\begin{proposition}[Stability in the slicing measure]
\label{prop:stability}
Let $\sigma,\sigma'$ be two slicing measures on $\Cont$ with finite sup-norm
moments of order $2(p-1)\vee 2$, and let $\VSW_p^{(\sigma)},\VSW_p^{(\sigma')}$ denote the corresponding distances. Then for any coupling $\Xi$ of $(\sigma,\sigma')$,
\[
\big|\VSW_p^{(\sigma)}(\bbP,\bbQ)^p-\VSW_p^{(\sigma')}(\bbP,\bbQ)^p\big|
\;\le\;2p\;\E_{(f,g)\sim\Xi}\Big[\big(2\max(\norm f_\infty,\norm g_\infty)\big)^{p-1}\,\norm{f-g}_\infty\Big].
\]
In particular, if a sequence of finite-rank approximations $f_R\sim\sigma_R$ can be coupled to $f\sim\sigma$ so that $\E\norm{f_R-f}_\infty^2\to0$ and $\sup_R\E\norm{f_R}_\infty^{2(p-1)}<\infty$ (both automatic from $L^2$ convergence when $p\le2$), then $\VSW_p^{(\sigma_R)}\to\VSW_p^{(\sigma)}$ uniformly on $\PPO\times\PPO$. Lemma~\ref{lem:kl} verifies this hypothesis for truncated Karhunen--Lo\`eve expansions.
\end{proposition}

\begin{proof}
Write $h(f):=\W_p\big((P_f)_\pf\bbP,(P_f)_\pf\bbQ\big)\in[0,2\norm f_\infty]$. By Lemma~\ref{lem:regularity}(ii), $|h(f)-h(g)|\le2\norm{f-g}_\infty$, and $|h(f)^p-h(g)^p|\le p\max(h(f),h(g))^{p-1}|h(f)-h(g)|$. Take expectations under $\Xi$ and use $\max(h(f),h(g))\le2\max(\norm f_\infty,\norm g_\infty)$; the marginals reproduce the two sliced distances. The uniformity claim follows by Cauchy--Schwarz on the right-hand side.
\end{proof}

\begin{lemma}[KL truncations satisfy the coupling hypothesis]
\label{lem:kl}
Let $\kappa$ satisfy Assumption~\ref{ass:kernel}, let $(\lambda_k,\varphi_k)_{k\ge1}$ be the Mercer eigenpairs of $\kappa$ on $L^2(\Om,\iota)$ for a fully supported finite Borel measure $\iota$ on $\Om$, and set
\[
f_R\;=\;\sum_{k=1}^{R}\sqrt{\lambda_k}\,\xi_k\,\varphi_k,
\qquad \xi_k\stackrel{\text{iid}}{\sim}\mathcal N(0,1).
\]
Then $\{\sqrt{\lambda_k}\varphi_k\}_{k\ge1}$ is an orthonormal basis of the Cameron--Martin space $\Hk$ of $\sigma$, the series converges almost surely in $(\Cont,\norm{\cdot}_\infty)$ with limit $f$ distributed as $\sigma$, and $\E\norm{f-f_R}_\infty^2\to0$. Consequently $\VSW_p^{(\sigma_R)}\to\VSW_p^{(\sigma)}$ uniformly on $\PPO\times\PPO$ by Proposition~\ref{prop:stability}.
\end{lemma}

\begin{proof}
By Mercer's theorem on the compact space $\Om$ (with $\kappa$ continuous and $\iota$ fully supported), the eigenfunctions $\varphi_k$ are continuous and $\{\sqrt{\lambda_k}\varphi_k\}$ is an orthonormal basis of $\Hk$, which coincides with the Cameron--Martin space of the Gaussian measure $\sigma$ on $\Cont$ \citep[Ch.~2--3]{bogachev1998gaussian}. For any orthonormal basis $\{e_k\}$ of the Cameron--Martin space of a centered Radon Gaussian measure on a separable Banach space, the series $\sum_k\xi_ke_k$ with i.i.d.\ standard Gaussian coefficients converges almost surely in the Banach norm, with limit distributed as the measure \citep[Thm.~3.5.1]{bogachev1998gaussian}; this is an It\^o--Nisio-type theorem. Finally, almost surely convergent Gaussian series in a separable Banach space converge in $L^r$ for every $r<\infty$ \citep[Ch.~3]{ledoux1991probability}, so $\E\norm{f-f_R}_\infty^2\to0$; the coupling $\Xi=\law(f_R,f)$ in Proposition~\ref{prop:stability} together with Cauchy--Schwarz gives the uniform convergence of the sliced distances.
\end{proof}

\section{Ablation}
\label{app:ablation}
\noindent\textbf{Ablation on $R$.} The number of random Fourier features $R$ is
specific to VSW. We
therefore sweep $R\in\{8,16,\dots,1024\}$ in the dataset-distance setting of
Figure~\ref{fig:corr} on FashionMNIST, at the two largest matched budgets
$L=S=10^3$ and $L=S=10^4$ (Table~\ref{tab:r_ablation}), with the baseline
configurations of
Appendix~\ref{app:impl}. Each entry is the
mean $\pm$ standard deviation over $5$ independent draws of the features and slices,
for VSW and both baselines.
Runtimes follow the protocol of Figure~\ref{fig:otdd_runtime} and are given on the
$30$ pairs at their own sizes, which is the cost of the correlation beside them, and
at fixed sizes $10^4$ and $6\times10^4$. Larger $R$ generally improves
correlation at increased runtime, with limited gains beyond $R=512$.
For slicing budget $L=S=1000$, all tested $R\ge16$ give higher mean
Pearson and Spearman correlations than DSW and s-OTDD; $R=8$ is below DSW.
At $L=S=10^4$, $R=128$ is competitive with DSW (Pearson
$0.8468\pm0.0080$ versus $0.8396\pm0.0125$), with overlapping variability
across random draws. We use $R=128$ for its balance of correlation and runtime.

\begin{table}[!ht]
\centering
\caption{Ablation on the number of random Fourier features $R$ on FashionMNIST, at
the two matched slicing budgets $L=S=10^3$ (top) and $L=S=10^4$ (bottom). Each entry
is the mean $\pm$ standard deviation over $5$ independent draws of the features and
slices. Times are per distance evaluation: on the $30$ pairs at their own sizes, at a
fixed size of $10^4$, and on the full training set of $6\times10^4$.}
\label{tab:r_ablation}
\footnotesize
\setlength{\tabcolsep}{5pt}
\renewcommand{\arraystretch}{1.0}
\begin{tabular}{l cc ccc}
\toprule
\multirow{2}{*}{Method}
& \multicolumn{2}{c}{Correlation with OTDD ($\uparrow$)}
& \multicolumn{3}{c}{Time per evaluation (ms, $\downarrow$)} \\
\cmidrule(lr){2-3} \cmidrule(lr){4-6}
& Pearson & Spearman & $5$--$10$k & $10^4$ & $6\times10^4$ \\
\midrule
\multicolumn{6}{l}{\textit{Matched slicing budget $L=S=10^3$}} \\
VSW, $R=8$             & $0.6030 \pm 0.0548$ & $0.6015 \pm 0.0695$ & $3.4$  & $3.2$  & $12.8$ \\
VSW, $R=16$            & $0.7931 \pm 0.0362$ & $0.7750 \pm 0.0519$ & $3.6$  & $3.8$  & $17.7$ \\
VSW, $R=32$            & $0.7698 \pm 0.0393$ & $0.7623 \pm 0.0653$ & $4.6$  & $4.8$  & $25.3$ \\
VSW, $R=64$            & $0.8044 \pm 0.0164$ & $0.8174 \pm 0.0309$ & $5.6$  & $6.5$  & $36.4$ \\
VSW, $R=128$  & $0.8413 \pm 0.0115$ & $0.8575 \pm 0.0151$ & $8.2$ & $10.4$ & $61.4$ \\
VSW, $R=256$           & $0.8525 \pm 0.0049$ & $0.8767 \pm 0.0194$ & $12.7$ & $15.0$ & $130.2$ \\
VSW, $R=512$           & $0.8730 \pm 0.0116$ & $0.8878 \pm 0.0051$ & $22.1$ & $25.9$ & $224.3$ \\
VSW, $R=1024$          & $0.8803 \pm 0.0049$ & $0.8897 \pm 0.0117$ & $54.9$ & $75.5$ & $401.0$ \\
\cmidrule(lr){1-6}
DSW                    & $0.7215 \pm 0.0373$ & $0.7356 \pm 0.0359$ & $46.9$  & $47.7$  & $141.9$ \\
s-OTDD                 & $0.2813 \pm 0.1932$ & $0.2948 \pm 0.1429$ & $291.7$ & $388.6$ & $2403.4$ \\
\midrule
\multicolumn{6}{l}{\textit{Matched slicing budget $L=S=10^4$}} \\
VSW, $R=8$             & $0.5971 \pm 0.0676$ & $0.6032 \pm 0.0817$ & $4.5$   & $4.9$   & $13.6$ \\
VSW, $R=16$            & $0.7891 \pm 0.0374$ & $0.7733 \pm 0.0577$ & $5.8$   & $6.3$   & $21.0$ \\
VSW, $R=32$            & $0.7723 \pm 0.0421$ & $0.7691 \pm 0.0734$ & $8.6$   & $9.4$   & $29.5$ \\
VSW, $R=64$            & $0.8056 \pm 0.0173$ & $0.8145 \pm 0.0322$ & $14.2$  & $15.2$  & $44.5$ \\
VSW, $R=128$  & $0.8468 \pm 0.0080$ & $0.8629 \pm 0.0095$ & $25.8$ & $27.2$ & $79.0$ \\
VSW, $R=256$           & $0.8610 \pm 0.0104$ & $0.8842 \pm 0.0203$ & $45.3$  & $43.8$  & $158.8$ \\
VSW, $R=512$           & $0.8768 \pm 0.0070$ & $0.8893 \pm 0.0116$ & $84.3$  & $82.3$  & $289.0$ \\
VSW, $R=1024$          & $0.8789 \pm 0.0031$ & $0.8938 \pm 0.0124$ & $158.9$ & $175.2$ & $527.5$ \\
\cmidrule(lr){1-6}
DSW                    & $0.8396 \pm 0.0125$ & $0.8385 \pm 0.0151$ & $90.6$   & $104.4$  & $438.2$ \\
s-OTDD                 & $0.4897 \pm 0.0670$ & $0.5305 \pm 0.0608$ & $2955.9$ & $3889.8$ & $24193.5$ \\
\bottomrule
\end{tabular}
\end{table}

\section{Implementation Details}
\label{app:impl}

Below, we provide experimental details for all of the reported experiments in the paper. All experiments are conducted on an AMD EPYC 7702P CPU, or NVIDIA RTX A5000 GPUs. Unless stated otherwise, VSW uses the Gaussian RBF kernel $\kappa_\ell(x,y)=\exp\!\big(-\|x-y\|^2/(2\ell^2)\big)$, whose spectral measure in \eqref{eq:rff} is $\tau_{\kappa_\ell}=\mathcal{N}(0,\ell^{-2}I_d)$. The bandwidth $\ell>0$ is set by the median heuristic~\citep{gretton2012kernel}. We pool the points of all inner measures of the two meta-measures being compared, draw $2000$ of them uniformly at random, and set $\ell$ to the median of their pairwise Euclidean distances. The bandwidth is computed once, before the slices are drawn.

\noindent\textbf{Statistical Properties.}
Each meta-measure is a mixture of $3$ Gaussian inner measures with uniform weights. The meta-measure is $\bbP=\frac13\sum_{k=1}^{3}\delta_{\mathcal{N}(\mu_k,I_d)}$, a uniform distribution over three Gaussian inner measures whose means $\mu_k$ are equally spaced on a circle of radius $4$ in the first two coordinates. To sample an empirical meta-measure, we draw $N$ labels $k_1,\dots,k_N$ independently and uniformly from $\{1,2,3\}$, and observe each inner measure $\mathcal{N}(\mu_{k_i},I_d)$ through $n$ i.i.d.\ points. Each empirical meta-measure has $N$ inner measures observed through $n$ points each. VSW uses $R=256$, a median-heuristic bandwidth, and every curve in Figure~\ref{fig:stats} is the mean over $200$ runs. For the Monte Carlo rate, we compare two independent samples with $N=n=1000$ and report $|\widehat{\VSW}_{2,L}-\widehat{\VSW}_{2,L_{\max}}|$ for $L\le 10^4$, where the reference $L_{\max}=10^5$ uses the same slices. For the outer and inner rates, both meta-measures are drawn from the same distribution, and we use $L=10^4$. The outer rate varies $N$ with $n=1000$. The inner rate varies $n$ with $N=1000$, and both meta-measures share the same $N$ component labels, so that only the inner sampling error remains.

\noindent\textbf{Optimal Transport Dataset Distance.}
Following \citet{sotdd}, we draw $30$ random pairs of disjoint sub-datasets from the training sets of MNIST, FashionMNIST, and CIFAR-10. For each pair, we sample a size uniformly from $\{5000,6000,\dots,10000\}$ and take the first and the last that many indices of a random permutation of the training set. Each sub-dataset is a meta-measure whose inner measures are its class-conditional distributions of flattened images. As a reference, we compute the exact OTDD~\citep{alvarez2020otdd} with its default implementation. Following \citet{dsw}, we set the weight of its feature cost to zero, so that OTDD reduces to a Wasserstein-over-Wasserstein distance between the class-conditional distributions. s-OTDD~\citep{sotdd} and DSW~\citep{dsw} are used with their default hyper-parameters. We use a vectorized implementation of DSW that returns the same values as the released code under identical random draws, while being optimized so that the runtime comparison remains fair. VSW uses $R=128$, with a bandwidth set once per pair of sub-datasets by the median heuristic.

The correlations in Figure~\ref{fig:corr} pool all $30$ pairs across the sampled
sizes. They are not conditioned on sample size, so shared sample-size
dependence can contribute to the reported agreement. This protocol introduces
no controlled class-proportion or domain shifts.

For Figure~\ref{fig:otdd_runtime}, we time one distance evaluation between two sub-datasets whose size ranges from $1000$ to the full training set in steps of $5000$, for $L=S\in\{10^3,5\times10^3,10^4\}$. All methods run on CPU. We report the mean and standard deviation over $10$ random draws of the sub-datasets.

\noindent\textbf{Comparing Distributions of Point Clouds.}
We follow the protocol and released code of \citet[Sec.~5.3]{dsw}. The reference meta-measure contains $N=10$ training shapes of one class, and the target contains $M$ test shapes of the same class. We vary one factor at a time: (i) the number of target shapes $M\in\{1,\dots,19\}$ (class \emph{bed}, $n=m=50$); (ii) the standard deviation $\sigma_{\mathrm{noise}}\in[0,0.5]$ of Gaussian noise added to the target points (class \emph{sofa}, $M=10$, $n=m=50$); and (iii) the target resolution $m\in\{50,100,\dots,600\}$ (class \emph{chair}, $M=10$, $n=500$). Each setting is repeated over $5$ random draws of shapes. VSW uses $L=2000$ slices with $R=128$, and bandwidth is set by the median heuristic. For OT-NNA~\citep{yang2019pointflow}, WoW, and DSW, we use the implementations and hyperparameters of \citet{dsw}. Our only change is a vectorized DSW implementation, which ensures optimal implementation of DSW and fair runtime comparison. Under identical random draws it returns the same values as the released code, so it reduces the runtime of DSW without changing its values. In Figure~\ref{fig:pc_toy}, each curve is divided by its mean over the sweep, and shaded regions show $\pm1$ standard deviation over the repetitions.

For Table~\ref{tab:pc_time}, every method compares the full sets of $662$ ShapeNet test chairs~\citep{Chang2015ShapeNetAI} and $662$ PVD samples~\citep{zhou20213pvd}, with $2048$ points per cloud. All methods use the settings above and run on one NVIDIA RTX A5000 GPU. For VSW and DSW, we report the mean $\pm$ standard deviation over $6$ timed batches, while OT-NNA and WoW are timed once.

\noindent\textbf{Comparing Image Distributions via Patch Distributions.}
We follow the protocol and released code of \citet[Sec.~5.4]{dsw}. Each meta-measure is a batch of $32$ grayscale $64\times64$ Perlin-noise images (scale $100$, pixel values rescaled to $[0,1]$), and each image is represented by the uniform empirical measure of its $3249$ overlapping $8\times8$ patches in $\mathbb{R}^{64}$. In the persistence sweep, the reference batch has persistence $0.5$ (lacunarity $2.5$, $5$ octaves) and the target persistence varies over $\{0.1,0.25,0.5,0.75,1\}$. In the lacunarity sweep, the reference has lacunarity $1.5$ (persistence $1$, $6$ octaves) and the target lacunarity varies over $\{1,1.25,1.5,1.75,2\}$. Reference and target batches are generated from different random seeds, each sweep is repeated $5$ times, and all methods are evaluated on the same pairs of batches. VSW uses $L=2000$ slices with $R=128$ random Fourier features of a Gaussian RBF kernel, whose bandwidth is set by the median heuristic on the pooled patches. WoW and DSW are implemented as in the point-cloud experiment above. Following the released code, DSW uses $100\times100$ slices, a $50$-point quantile grid, and length scale $0.1$, and we again use a vectorized implementation that returns the same values as the released code. In Figure~\ref{fig:patch}, each curve is divided by its maximum within a run before averaging, which preserves the depth of its minimum at the reference parameter. Shaded regions show $\pm1$ standard deviation over the $5$ runs. The runtimes in Figure~\ref{fig:patch} are for one distance evaluation between two batches. For VSW and DSW, each evaluation is repeated $5$ times, and the mean is averaged over all $50$ pairs of batches. All runtimes are measured on the CPU.

\noindent\textbf{Gradient Flow on Natural Images.}
The target meta-measure $\bbQ$ consists of $N=5$ ImageNet classes~\citep{imagenet} (goose, koala, Persian cat, monarch, and macaw), each represented by the uniform empirical measure of its first $20$ training images. Each image is resized and centre-cropped to $64\times64$, scaled to $[0,1]$, and flattened, so that $\Om\subset\R^{12288}$. The source meta-measure has the same shape, with $N=5$ clouds of $n=20$ particles drawn i.i.d.\ from $\mathcal{N}(0,I)$. We run the particle scheme of Section~\ref{sec:flows}  for $110{,}000$ steps. At every step, we draw fresh slices ($L=8000$ slices, $R=4096$ random Fourier features of a Gaussian RBF kernel) and a Monte Carlo estimate of the gradient in Eq.~\eqref{eq:gradient}  by automatic differentiation. Since each particle carries mass $1/(Nn)$, we multiply the gradient by $Nn$ so that the step size does not depend on the number of particles. The particles are updated by gradient descent with heavy-ball momentum $0.9$, and the momentum is reset when the bandwidth changes. The bandwidth follows a two-phase schedule. The large bandwidth first moves each cloud towards a target cloud, and the smaller one then resolves individual images. The full run takes about $6$ minutes on one GPU.

In Figure~\ref{fig:grad_flow} (left), each row follows one particle of one class, and the last column shows its nearest ground-truth image of the matched class in Euclidean distance. For the right panel, we clip all snapshots to $[0,1]$, embed them with CLIP ViT-L/14~\citep{radford2021learning} using its standard preprocessing, and $\ell_2$-normalize the embeddings. We then project them with a 2D PCA fitted on the $100$ ground-truth embeddings only, so that all snapshots share the same axes. Each line traces one particle through the snapshots at steps $0$, $1$k, $2$k, $3$k, $5$k, $8$k, $15$k, $25$k, and $110$k.

\noindent\textbf{Parametric Generative Modeling.}
DSW~\citep{dsw} uses $100\times100$ slices and a $50$-point quantile grid, with its GP length scale tuned per dataset ($0.03$ and $0.003$). MMDSW~\citep{bonet2025flowing} is the MMD with the sliced-Wasserstein Riesz kernel $-\mathrm{SW}_2^r$, with $128$ projections and $r=1$ and $r=0.5$, respectively. For DSW~\citep{dsw}, we again use a vectorized implementation that returns the same values as the reference one. We report the point-level $\W_2^2$ between $4096$ generated points, one from each of $4096$ generated clouds, and $4096$ points from the target, computed with POT~\citep{flamary2021pot} and averaged over $2$ draws. The ground-truth level is $\W_2^2$ between two independent target samples of the same size. Curves show the mean and standard deviation over $3$ seeds, and the scatter plots in Figure~\ref{fig:cb_gen} show seed $1$. The wall-clock is measured on GPU.

\noindent\textbf{Shape Classification.}
We follow the protocol of \citet[Sec.~5.1]{dsw} and use their released implementations of all baselines. Each shape is a metric-measure space given by its pairwise distance matrix, which we map to a meta-measure on $\Om\subset\R$: the outer measure is uniform over the points, and the inner measure of point $i$ is the empirical distribution of its distances to all points, i.e., row $i$ of the distance matrix. For every method, we compute the full matrix of distances between shapes and report the mean and standard deviation of the $3$-NN accuracy over $1000$ random stratified $25\%/75\%$ train/test splits. FAUST-500 and FAUST-1000 use the $100$ MPI-FAUST training registrations~\citep{faust}. For 2D Shapes~\citep{beier2022gromov} and Animals~\citep{sumner2004meshes}, we use the preprocessed data released by \citet{beier2022gromov}. For MNIST-2000, we follow the construction of \citet{dsw}: for each digit $0$--$4$, we build $20$ shapes by sampling $2000$ points from the mean training image proportionally to pixel intensity, adding Gaussian noise with standard deviation $3$ pixels, and applying a random rotation, with Euclidean distances.

VSW uses $L=100$ and $R=16$. SQW and STLB use $100$ projections and a $10$-point quantile grid (length scale $0.01$ for SQW), and TLB, AE, and GW are used as in the released code, without a feature term. To compare the methods fairly, VSW, SQW, and STLB each draw a single random basis for the whole distance matrix, so that every shape is projected once rather than once per pair. AE caches its per-shape energy self-terms, which changes its distances by less than $0.2\%$. TLB and GW are deterministic and admit no such amortization. Runtimes in Table~\ref{tab:shape_classification} are the mean time per distance on CPU.
\end{document}